\documentclass[conference]{IEEEtran}
\usepackage{microtype}
\usepackage{graphicx}
\usepackage{subfigure}
\usepackage{booktabs} 
\usepackage{bbm,bm}
\usepackage{epsfig,ifthen}
\usepackage{latexsym}
\usepackage{amsfonts}
\usepackage{eepic}
\usepackage{amsopn}
\usepackage{amsmath,amssymb,amsthm,rotating,graphicx,rotating,float}
\usepackage{accents}
\usepackage[mathscr]{eucal}
\usepackage{url}
\usepackage{graphicx}
\usepackage{csquotes}
\usepackage{verbatim}
\usepackage{ulem}
\usepackage{algorithm,algorithmic}
\usepackage{color} 
\newtheorem{theorem}{Theorem}[section]
\newtheorem{lemma}[theorem]{Lemma}

\newtheorem{corollary}[theorem]{Corollary}
\newcommand{\eat}[1]{}
\newenvironment{definition}[1][Definition]{\begin{trivlist}
\item[\hskip \labelsep {\bfseries #1}]}{\end{trivlist}}

\newenvironment{remark}[1][Remark]{\begin{trivlist}
\item[\hskip \labelsep {\bfseries #1}]}{\end{trivlist}}

\DeclareMathOperator*{\argmax}{argmax}

\def \w {{\mathbf{w}}}
\def \x {{\mathbf{x}}}
\def \y {{\mathbf{y}}}
\def \z {{\mathbf{z}}}

\def \s {{\mathbf{s}}}

\def \g {{\mathbf{g}}}

\def \bmu {{\boldsymbol\mu}}
\def \E {{\mathbb{E}}}

\def \X {{\mathcal{X}}}
\def \K {{\mathcal{K}}}
\def \bv {\mathbf{v}}
\def \bzeta {\boldsymbol{\zeta}}
\def \bzetaL {\bzeta^{(L)}}
\def \bzetaLS {\bzeta^{(L \cup S)}}
\def \bzetaLj {\bzeta^{(L \cup \{j\})}}
\def \bzetaLjnot {\bzeta^{(L \cup \{j_0\})}}

\def \bzetaLv {\bzeta^{(L \cup \{v\})}}
\def \lbzetaL {l\left(\bzetaL \right)}
\def \lbzetaLS {l\left(\bzetaLS \right)}
\def \lbzetaLj {l\left(\bzetaLj \right)}

\def \gradljbzetaL {\nabla l_{j} \left(\bzetaL \right)}

\def \gradlvpbzetaL {\nabla l_{v}^+ \left(\bzetaL \right)}
\def \gradljpbzetaL {\nabla l_{j}^+ \left(\bzetaL \right)}
\def \R {\mathbb{R}}
\def \bonej {\mathbf{1}^{(\{j\})}}
\def \bonev {\mathbf{1}^{(\{v\})}}

\def \byj {\by^{(\{j\})}}
\def \byv {\by^{(\{v\})}}

\newcommand{\by}{\mathbf{y}}

\newcommand{\bw}{\mathbf{w}}

\begin{document}
\title{Efficient Data Representation by Selecting Prototypes with Importance Weights}

\author{\IEEEauthorblockN{Karthik S. Gurumoorthy}
\IEEEauthorblockA{\textit{Amazon Development Center} \\
Bangalore, India \\
gurumoor@amazon.com}
\and
\IEEEauthorblockN{Amit Dhurandhar}
\IEEEauthorblockA{\textit{IBM Research} \\
Yorktown Heights, U.S.A. \\
adhuran@us.ibm.com}
\and
\IEEEauthorblockN{Guillermo Cecchi}
\IEEEauthorblockA{\textit{IBM Research} \\
Yorktown Heights, U.S.A. \\
gcecchi@us.ibm.com}
\and
\IEEEauthorblockN{Charu Aggarwal}
\IEEEauthorblockA{\textit{IBM Research} \\
Yorktown Heights, U.S.A. \\
charu@us.ibm.com}
}
\maketitle

\begin{abstract}
Prototypical examples that best summarizes and compactly represents an underlying complex data distribution communicate meaningful insights to humans in domains where simple explanations are hard to extract. In this paper we present algorithms with strong theoretical guarantees to mine these data sets and select prototypes a.k.a. representatives that optimally describes them.
Our work notably generalizes the recent work by Kim et al. (2016) where in addition to selecting prototypes, we also associate non-negative weights which are indicative of their importance. This extension provides a single coherent framework under which both prototypes and criticisms (i.e. outliers) can be found. Furthermore, our framework works for any symmetric positive definite kernel thus addressing one of the key open questions laid out in Kim et al. (2016).\eat{Though our additional requirement of learning non-negative weights no longer maintains submodularity of the objective as in the previous work,} By establishing that our objective function enjoys a key property of that of weak submodularity, we present a fast ProtoDash algorithm and also derive approximation guarantees for the same. We demonstrate the efficacy of our method on diverse domains such as retail, digit recognition (MNIST) and on publicly available 40 health questionnaires obtained from the Center for Disease Control (CDC) website maintained by the US Dept. of Health. We validate the results quantitatively as well as qualitatively based on expert feedback and recently published scientific studies on public health, thus showcasing the power of our technique in providing actionability (for retail), utility (for MNIST) and insight (on CDC datasets) which arguably are the hallmarks of an effective data mining method.
\end{abstract}


\section{Introduction}
\eat{Interpretable modeling has received a lot of attention in recent times \cite{lime,Kim16,irt,decl,twl}. The reason being that nearly every real application with a human making decisions at its helm needs to have confidence in the model before its judgment can be trusted. Interestingly, interpretability has also become important in deep learning \cite{gan} given all the recent studies \cite{carlini,gan} showing their susceptibilities to slightly perturbed adversarial examples. A successful approach in understanding human evaluations has been to explain decisions made by an application by mining important and influential data points or features that best describe these decisions}

A successful approach in understanding real world data on which machine learning models are built to enable automated decision making, is to extract important and influential data points or features that best describes the underlying data
 \cite{Kim16,koh17,weiser1982programmers}. These approaches can be unified as finding a subset $S$ out of a collection $V$ of items (data points, features, etc.) that maximize a scoring function $f(S)$. The scoring function measures the information, relevance and quality of the selection. It may also discourage redundancy to obtain compact, informative subsets. Such subset selection problems are predominately useful when summarizing data sets to offer data scientists a first impression of the scope of a data set, in identifying outliers, and for compressing training data sets to accelerate the training of data-hungry deep learning methods. The desiderata for the scoring function naturally imply notions of \emph{diminishing returns}: for any two sets $S \subset T \subset V$ and any item $i \notin T$, it holds that $f(S \cup \{i\}) - f(S) \geq f(T \cup \{i\}) - f(T)$. A scoring function satisfying this diminishing return property is called a \emph{submodular} function \cite{fujishige05,lo83}. Importantly, submodularity often implies tractable algorithms with good theoretical guarantees.

In this paper we provide two algorithms for selecting prototypical examples from complex datasets. By showing that our scoring function $f(.)$ satisfies a key property of \emph{weak submodularity} \cite{weaksubInit}, we derive strong theoretical bounds for our selection methods. Loosely speaking, weak submodularity is a class of approximately submodular functions. The \emph{weak} part in these approximate submodular functions are precisely defined in terms of their submodularity ratio as stated in (\ref{subr1}). Showing that our problem is weakly submodular immediately leads to a standard greedy algorithm which we call ProtoGreedy. Our main contribution is a faster yet theoretically sound algorithm called ProtoDash for which we derive approximation guarantees. Our work builds on top of the Learn to Criticize method (L2C) \cite{Kim16} where the authors provide an approach to select prototypical (as well as outlying) examples from a given complex dataset for a pre-specified sparsity level $m$. We generalize this work to not only select prototypes for a given sparsity level $m$ but also to associate non-negative weights with each of them indicative of the importance of each prototype. This extension leads to multiple advantages over L2C: a) the weights allow for assessing the importance of the prototypes, b) the non-negativity aids in making this comparison more natural and hence more easy to interpret \cite{nmf}, c) it provides a single coherent framework under which both prototypes and criticisms -- which are the farthest (or least weighted) examples from our prototypes -- can be found and d) our framework works for any symmetric positive definite kernel which is not the case for L2C.

\eat{Our requirement of learning weights for prototypes does not maintain the submodularity of the objective as in L2C.} Additionally, we see our work as addressing one of the open questions laid out in \cite{Kim16} and we quote, \emph{"For future work, we hope to further explore the properties of L2C such as the effect of the choice of kernel, and weaker conditions on the kernel matrix for submodularity."} To this end, we show that having unequal weights for the prototypes eliminates any additional conditions on the kernel matrix but at the expense of forgoing submodularity. However, we show that the set function is still weakly submodular for which we present a greedy algorithm ProtoGreedy and a fast ProtoDash algorithm. Our main algorithm ProtoDash is much faster that ProtoGreedy both in theory and in practice. Though it has slightly loose theoretical bounds compared to ProtoGreedy, their performance in practice are virtually indistinguishable as observed in our experiments. We provide approximation guarantees for both these methods as well as analyze their execution time. 

Furthermore, ProtoDash not only finds prototypical examples for a dataset $X$, but it can also identify (weighted) prototypical examples from $X^{(2)}$ that best represent another dataset $X^{(1)}$ which may be \emph{different} from $X^{(2)}$.\eat{ where $X^{(1)}$ and $X^{(2)}$ belong to the same feature space.} This aspect has applications to settings associated with covariate shift \cite{Agarwal11} where the non-negative weights are computed for the samples in the Treatment set ($X^{(2)}$) so that the weighted samples best approximate the different Control distribution ($X^{(1)}$).\eat{weights associated with the chosen samples are computed only for the $m$ prototypes.} Further, many big companies like Amazon have highly skewed datasets where the number of regular members are orders of magnitude more than Prime members. In such cases, building a model to predict say expenditure using all the data will lead to very bad predictions for the Prime members. Usually, learning on a randomly sampled subset from the regular members which is roughly the size of the Prime members along with the Prime members data leads to much improved performance. Our (deterministic) method ProtoDash which can also be used to select prototypes across datasets can be very useful in such cases. We in fact showcase the power of our method in the experiments where on a large retail dataset, the prototypes actually improve performance over using all the data as well as random subsampling and other baselines. We also present our methods efficacy on MNIST where we gradually skew the distribution of the target set $X^{(1)}$ from it being a representative sample of the original dataset (i.e. approximately equal \# of 0s, 1s, 2s, ..., 9s) to containing only a single digit, with $X^{(2)}$ the source set remaining unchanged. Our method naturally adapts to such skewness in distributions by picking more (and increased weight) representatives of the skewed digit in $X^{(1)}$ from $X^{(2)}$ leading to a significantly better performance. In addition, our extensions induce an implicit metric that can be used to order $k$ different datasets $X_1,...,X_k$ based on how well their prototypes represent $X^{(1)}$. This aspect is used to create a directed graph based on the 40 health questionnaires available through Center for Disease Control (CDC). The graph depicts the relationship between questionnaires in terms of how well they represent each other. Such connections can be used to find surrogates or even further study causal relationships between the conditions/categories denoted by these questionnaires. For instance, the intriguing finding of our method that Early Childhood most affects Income is validated by a recent study \cite{poor20}. We can thus obtain socially influential insights at low cost which could lead to deeper investigations in the future.
\eat{For instance, our method finds that the income of an individual is primarily affected by his early childhood and occupation. Occupation is natural to think of, but it is interesting that early childhood was selected as the most important factor given 39 other possibilities. However, this interesting insight can be justified by a recent article in the Atlantic \cite{poor20}, which talks about the significant decrease in social mobility in the 2000s. We can thus obtain and recover socially impactful insights at low cost, which could be a starting point for deeper investigations in the future.}

As learning with non-negativity constraints \eat{for interpretability} is strongly enforced in the recommender systems literature \cite{nmf} as well as in classification and regression domains \cite{nonnegcl}, we impose similar constraints on our prototype weights. In our setting it is absolutely necessary as selecting an example to be a prototype for a dataset and then associating a negative weight to it would be very confusing (intuitively inconsistent) for a user to understand. Such non-negativity constraints in turn precludes us from directly borrowing the results of \cite{weaksub} or \cite{Kim16}. We need to explicitly prove that the set function is still weakly submodular with the added non-negativity constraint and reestablish all the guarantees derived in \cite{weaksub} as their original arguments cannot be directly used. We discuss this crucial point in Section~\ref{sec:relatedwork}.

We briefly summarize our important contributions in this work:
\begin{itemize}
\item We present a single coherent framework to select both prototypes and criticisms that compactly represents underlying data.
\item By having our framework work for any symmetric positive definite kernel, we address one of the key open questions laid out in \cite{Kim16}.
\item We establish the weak submodular property of our scoring function even with additional non-negativity constraint and present tractable algorithms \emph{with theoretical guarantees} to optimize it. Our algorithm of choice is ProtoDash as it much faster than ProtoGreedy both in theory and practice without compromising on the quality of the solution.
\item By identifying prototypes that can best present a possibly different target set, we extend the application of our method to covariate shift settings.
\item Through our carefully designed experiments, we showcase the usefulness of our method in providing actionability, utility, and insight when summarizing the datasets.
\end{itemize}

\section{Problem Statement}
In the most general setting, let $X^{(1)}$ and $X^{(2)}$ represent the target and the source set respectively. We set prototypes from $X^{(2)}$ that best represents $X^{(1)}$ by maximizing the scoring function $f(.)$. In the special case when these datasets are the same, i.e. $X^{(1)} = X^{(2)}$, the selected prototypes summarizes the underlying data distribution.
Let $\X$ be the feature space from which we obtain the samples $X^{(1)}$ and $X^{(2)}$. Consider a kernel function $k:\mathcal{X} \times \mathcal{X} \rightarrow \mathbb{R}$ and its associated reproducing kernel Hilbert space (RKHS) $\K$ endowed with the inner product $k(\x_i,\x_j) = \langle \phi_{\x_i},\phi_{\x_j} \rangle$ where $\phi_{\x}(\y) = k(\x,\y) \in \K$ is continuous linear functional satisfying
$\phi_{\x}: h \rightarrow h(\x) = \langle \phi_{\x},h\rangle$
for any function $h \in\K : \X \rightarrow \mathbb{R}$.

The maximum mean discrepancy (MMD) \cite{Gretton06} is a measure of difference between two distributions $p$ and $q$ given by:
\begin{equation*}
\begin{split}
MMD(\K,p,q) &= \sup\limits_{h \in \K} \left(\E_{\x \sim p}[h(\x)] - \E_{\y \sim q}[h(\y)] \right)\\ &= \sup\limits_{h \in \K} \langle h, \bmu_p - \bmu_q\rangle
\end{split}
\end{equation*}
where $\bmu_p = \E_{\x \sim p}[\phi_{\x}]$.
\eat{
Rewriting the expectations as
\begin{align*}
\E_{\x \sim p}[h(\x)]  &= \E_{\x \sim p}[\langle \phi_{\x},h\rangle] = \langle h, \E_{\x \sim p}[\phi_{\x}] \rangle = \langle h, \bmu_p\rangle
\end{align*}
where $\bmu_p = \E_{\x \sim p}[\phi_{\x}]$, we get $MMD(\K,p,q) = \sup\limits_{h \in \K} \langle h, \bmu_p - \bmu_q\rangle$. As shown in \cite{Gretton06}, the supremum is achieved when 
\begin{equation}
\label{eq:witnessfunc}
h_{max}(\y) = \bmu_p(\y) - \bmu_q(\y) = \E_{\x \sim p}[k(\x,\y)] -  \E_{\x \sim q}[k(\x,\y)]
\end{equation}
and so we have
\begin{equation*}
MMD(\K,p,q) = \parallel h_{max} \parallel^2 = \E_{\x,\y \sim p}[k(\x,\y)] +  \E_{\x,\y \sim q}[k(\x,\y)] - 2  \E_{\x \sim p,\y \sim q}[k(\x,\y)].
\end{equation*}
The supremum function in (\ref{eq:witnessfunc}) is called the \emph{witness} function \cite{Kim16}.
}

Our goal is to approximate $\bmu_p$ by a weighted combination of $m$ sub-samples $Z\subseteq X^{(2)}$ drawn from the distribution $q$, i.e.,
$\bmu_p(\x) \approx \sum\limits_{j:\z_j \in Z} w_j k(\z_j,\x)$
\eat{which in turn implies $\sup\limits_{\parallel h \parallel \leq 1} \left(\E_{\x  \sim p}[h(\x)] - \langle h,\bmu_{q^{\prime}} \rangle\right) = \sup\limits_{\parallel h \parallel \leq 1} \langle h, \bmu_p-\bmu_{q^{\prime}} \rangle = \parallel \bmu_p-\bmu_{q^{\prime}} \parallel$ giving us tight approximations in expectations of bounded functions $h$. These weights are the RND of $q^{\prime}$ w.r.t $q$, i.e. $\beta_i = \frac{\,d q^{\prime}(\z_i)}{\,d q(\z_i)}$.}
\eat{Recall that our goal is to estimate $\beta_i$'s over a small set of $m$ samples in $X^{(2)}$ such that the empirical estimate of the expected value $\bmu_{q^{\prime}}$ closely approximates the empirical estimate of $\bmu_p$ computed using \emph{all} the samples in $X^{(1)}$.} 
where $w_j \ge 0$ is the associated weight of the sample $\z_j \in X^{(2)}$. We thus need to choose the prototype set $Z \subseteq X^{(2)}$ of cardinality ($|.|$) $m$ and learn the non-negative weights $w_j$ that minimizes the finite sample $MMD$ metric as given below:
\begin{equation}
\label{eq:MMDhat}
\begin{split}
&\widehat{MMD}(\mathcal{K},X^{(1)},Z,\w) \\&= \frac{1}{(n^{(1)})^2} \sum\limits_{\x_i,\x_j \in X^{(1)}} k(\x_i,\x_j) - \frac{2}{n^{(1)}} \sum\limits_{\z_j \in Z} w_j \sum\limits_{\x_i \in X^{(1)}} k(\x_i,\z_j) \\
&+ \sum\limits_{\z_i, \z_j \in Z} w_i w_j k(\z_i,\z_j);\text{    subject to }w_j \geq 0, \forall \z_j\in Z.
\end{split}
\end{equation}
Here $n^{(1)}=|X^{(1)}|$. Index the elements in $X^{(2)}$ from 1 to $n^{(2)}=|X^{(2)}|$ and for any $Z \subseteq X^{(2)}$ let $L_{Z} \subseteq \left[n^{(2)}\right]=\{1,2,\ldots, n^{(2)}\}$ be the set containing its indices. Discarding the constant terms in (\ref{eq:MMDhat}) which do not depend on $Z$ and $\w$, we define the function
\begin{equation}
\label{eq:l} 
l\left(\w\right) =  \w^T \bmu_p - \frac{1}{2} \w^T K \w
\end{equation}
where $K_{i,j} =  k(\y_i,\y_j)$ and $\mu_{p,j} = \frac{1}{n^{(1)}} \sum\limits_{\x_i \in X^{(1)}} k(\x_i,\y_j); \forall \y_j \in X^{(2)}$ is the point-wise empirical evaluation of the mean $\bmu_p$. Our goal then is to find a index set $L_Z$ with $\left|L_Z\right| \leq m$ and a corresponding $\w$ such that the set function $f: 2^{\left[n^{(2)}\right]} \rightarrow \mathbb{R}^+$ defined as
\begin{equation}
\label{def:f}
f\left(L_{Z}\right) \equiv \max\limits_{\w: supp(\w) \in L_{Z},\w \geq 0} l\left(\w\right)
\end{equation}
attains maximum. Here $supp(\w) = \{j: \w_j > 0\}$. We denote the maximizer for the set $L_Z$ by $\bzeta^{\left(L_Z\right)}$. 

It is important to note that as $f\left(L_{Z}\right) = l\left(\bzeta^{\left(L_Z\right)}\right)$, our interchangeable usage of any the notations means the same. So maximizing $f(.)$ in terms of finding the optimal set $L_Z$ is equivalent to maximizing $l(.)$ w.r.t. $L_Z$ where given any set $L$, the function $l(.)$ is always assumed to be evaluated at the maximizing location $\lbzetaL$. So our goal is to identify the optimal set $L_Z$. Once determined we set the prototype weights $\bw = \bzeta^{\left(L_Z\right)}$.

\section{Related Work}
\label{sec:relatedwork}
Recently, there has been a surge of papers proposing interpretable models motivated by diverse applications such as medical \cite{caruana}, information technology \cite{irt} and entertainment \cite{lime}. These strategies involve building rule/decision lists \cite{decl,twl}, to finding prototypes \cite{Kim16} in an unsupervised manner akin to our work or strictly in a supervised manner as \cite{sproto}, to taking inspiration from psychometrics \cite{irt} and learning understandable models. Works such as \cite{lime} differ from the above methods in that they focus on answering instance-specific user queries by locally approximating a superior performing complex model with a simpler, interpretable one. The hope is that the insights conveyed by the simpler model will be consistent with the complex model.

In our work, as mentioned above, we generalize the setting in \cite{Kim16} and propose algorithms that select prototypes with non-negative weights associated with them. On the technical side, one recent work that we leverage and extend with non-negativity constraints for our MMD objective is \cite{weaksub}. We recover their bounds even with the non-negativity constraint. In fact, our bounds are tighter since the restricted concavity parameter $c_\Omega$ and restricted smoothness parameter $C_\Omega$ stated in Definition~\ref{def:RSCRSM} are obtained by searching over only the non-negative orthant as opposed to the entire $\mathbb{R}^b$ space, where $b$ is feature space dimension. Moreover, given our specific functional form for the objective, we show in Corollary \ref{cor1} that choosing an element with the largest gradient in ProtoDash at each step is equivalent to maximizing a tight lower bound on $l(.)$, which is not necessarily true for the setting considered in \cite{weaksub}. Additionally, the gradient in our case can be easily computed. The added technical difficulty when deriving the guarantees in our case stems from the fact that we cannot let the gradients go to zero as the non-negativity constraints would make our solution infeasible. As a consequence, we cannot directly use the results of \cite{weaksub} or \cite{Kim16}. The complexity lies in showing that our set function remains to be weakly submodular even with the additional non-negativity constraints (required for assessing the importance of prototypes) as we prove in Theorem~\ref{ws}. Weak submodularity alone does not provide the bound for ProtoDash and we explicitly derive the theoretical guarantee in Theorem \ref{pd}. Lemma \ref{lem:KKTGradient} proved in our work is essential for proving both Theorems \ref{ws} and \ref{pd}, which is not the case in \cite{weaksub}.

\section{Prototype Selection Framework}
Algorithm~\ref{protogreedy} describes the steps involved in \emph{ProtoGreedy}. It is algorithmically similar to L2C described in \cite{Kim16} where both the methods greedily select the next element that maximizes the increment of the scoring function. Given the current set $L$, ProtoGreedy selects that element $j_0$ that produces the greatest increase in objective value $f(.)$, i.e. $j_0= \argmax_{j \notin L} f(L \cup j)-f(L)$. The key difference w.r.t. L2C is that ProtoGreedy additionally determines the (unequal) non-negative weights $\bzetaLjnot$ for each of the selected prototypes whereas in L2C the weights are \emph{set} to equal to $1/|L \cup {j_0}|$. Our main contribution w.r.t. ProtoGreedy is in showing that the set function is weakly submodular even with the additional non-negativity constraints on the weights based on revisiting concepts such as weak submodularity, restricted strong concavity (RSC) and restricted smoothness (RSM). We establish this property by proving that $f(.)$ is monotonic and its submodularity ratio $\gamma$ is bounded away from zero; implying that it is weakly submodular. The approximation guarantee of $\left(1-e^{-\gamma}\right)$ for ProtoGreedy then follows using the results from \cite{weaksub}.

Algorithm~\ref{protodash} describes our faster and desired algorithm \emph{ProtoDash}. In ProtoDash we choose an element $j_0$ whose \emph{gradient} given by $\mu_{p,j_0}-K_{j_0,*}\zeta^{(L)}$ is the highest over the set of candidates, i.e. $j_0 = \argmax_{j \notin L} \mu_{p,j}-K_{j,*}\zeta^{(L)}$. As the chosen element may not be the one with the highest increment in $f(.)$, the choices made by ProtoGreedy and ProtoDash in each iteration can be different. However, as is shown in Corollary \ref{cor1} the selected element
in ProtoDash does indeed maximize a tight lower bound on the increment of $f(.)$ highlighting the connection between the two algorithms. Once the best element is determined the optimal weights $\bzetaLjnot$ are computed as above. Unlike ProtoGreedy, the approximation guarantee for ProtoDash do not directly follow from \cite{weaksub} and we explicitly derive it in Theorem~\ref{pd}.

 The primary advantage of ProtoDash over ProtoGreedy is the computational speedup of two orders of magnitude as explained in Section~\ref{sec:timecomplexity}. While ProtoGreedy requires solving a quadratic program of time complexity $O(m^3)$ for each of the remaining $n^{(2)}-|L|+1$ elements to select the next best element, ProtoDash requires only a search over their gradient values each computable in $O(m)$, thereby leading to an $O(m^2)$ speedup during \emph{every} element search. For both algorithms the termination condition can either be a sparsity level $m$ or a minimal increase in objective value $\epsilon$ that is required for selecting more elements.

\begin{algorithm}[t]
    \caption{ProtoGreedy }
    \label{protogreedy}
\begin{algorithmic}
\STATE \textbf{Input:} sparsity level $m$ or lower bound $\epsilon$ on increase in $f(.)$, $X^{(1)}$, $X^{(2)}$
\STATE $L=\emptyset$
\WHILE{termination condition is false}\STATE\COMMENT{i.e., $|L|\le m$, else increase in objective value $\ge \epsilon$.} 
\STATE $\forall j \in \left[n^{(2)}\right] \setminus L, v_j = f\left(L \cup \{j\} \right) - f(L)$ 
\STATE $ j_0 = \argmax\limits_j v_j$ 
\STATE $L = L \cup \{j_0\}$ 
\STATE $\bzetaL=\argmax\limits_{\w:supp(\w)\in L,\w\ge 0} l(\w)$  
\ENDWHILE 
\RETURN $L$, $\bzetaL$
\end{algorithmic}
\end{algorithm}

\begin{algorithm}[t]
    \caption{ProtoDash}
    \label{protodash}
\begin{algorithmic}
\STATE \textbf{Input:} sparsity level $m$ or lower bound $\epsilon$ on increase in $f(.)$, $X^{(1)}$, $X^{(2)}$
\STATE $L=\emptyset$, $\bzetaL=\mathbf{0}$
and $\g=\nabla l(\mathbf{0})=\bmu_p$
\WHILE{termination condition is false}\STATE\COMMENT{i.e., $|L|\le m$, else increase in objective value $\ge \epsilon$.}  
\STATE $j_0=\argmax\limits_{j \in \left[n^{(2)}\right] \setminus L} g_j$
\STATE $L = L \cup \{j_0\}$
\STATE $\bzetaL=\argmax\limits_{\w:supp(\w)\in L,\w\ge 0} l(\w)$  
\STATE $\g=\nabla \lbzetaL = \bmu_p-K\bzetaL$
\ENDWHILE 
\RETURN $L$, $\bzetaL$ 
\end{algorithmic}
\end{algorithm}

\subsection{Preliminaries}
Given an integer $b >0$, let $[b]:=\{1,...,b\}$ denote the set of the first $b$ natural numbers. Let $\langle x,y\rangle$ denote dot product of vectors $x$ and $y$.

\begin{definition}[Definition 1 (Submodularity Ratio):]
Let $L, S \subset [b]$ be two disjoint sets, and $f:[b]\rightarrow R$. The submodularity ratio \cite{weaksubInit} of L with respect to (w.r.t.) S is given by:
\begin{equation}
\label{subr1}
\gamma_{L,S} = \frac{\sum_{i\in S}\left(f(L\cup i)-f(L)\right)}{f(L\cup S)-f(L)}
\end{equation}
The submodularity ratio of a set $U$ w.r.t. a positive integer $r$ is given by:
\begin{equation}
\label{subr2}
\gamma_{U,r} = \min\limits_{\substack{L,S: L \cap S=\emptyset \\ L\subseteq U;|S|\le r}}\gamma_{L,S}
\end{equation}
\end{definition}

The function $f(.)$ is submodular iff $\forall L,S$, $\gamma_{L,S}\ge 1$. However, if $\gamma_{L,S}$ can be shown to be bounded away from 0 but not necessarily $\ge 1$, then $f(.)$ is said to be weakly submodular.

\begin{definition}[Definition 2 (RSC and RSM):]
\label{def:RSCRSM}
A function $l:R^b\rightarrow R$ is said to be restricted strong concave (RSC) with parameter $c_\Omega$ and restricted smooth (RSM) with parameter $C_\Omega$ \cite{weaksub} if $\forall \x,\y \in \Omega\subset R^b$;

\begin{equation}
-\frac{c_\Omega}{2}\|\y-\x\|^2_2\ge l(\y)-l(\x)-\langle\nabla l(\x),\y-\x\rangle\ge-\frac{C_\Omega}{2}\|\y-\x\|^2_2.
\end{equation}
\end{definition}
We denote the RSC and RSM parameters on the domain $\Omega_m =\{ \x: \|\x\|_0 \leq m; \x \geq 0\}$ of all \emph{m-sparse non-negative} vectors by $c_{m}$ and $C_{m}$ respectively. \eat{We have, if $m_1 \leq m_2$, then $c_{m_1} \geq c_{m_2}$ and $C_{m_1} \leq C_{m_2}$.} Also, let $\tilde{\Omega} = \{(\x,\y): \|\x-\y\|_0 \leq 1\}$  with the corresponding smoothness parameter $\tilde{C}_1$. It can be easily verified that if $k \leq m$ then $c_k \geq c_m$ and $C_k \leq C_m$ as $\Omega_k \subseteq \Omega_m$.

\subsection{Theoretical Guarantees}
Based on the above two definitions we develop two algorithms for prototype selection and derive their approximation guarantees as established below. Please refer to Appendix for the detailed proofs.
\begin{lemma}[Monotonicity]
\label{mono}
The set function $f$ defined in (\ref{def:f}) is monotonic, meaning that if $L_1 \subseteq L_2$ then $f(L_1) \leq f(L_2)$.
\end{lemma}

\begin{lemma}[Finite RSC and RSM]
\label{lemma:RSCRSM}
Given a symmetric positive definite kernel matrix \cite{Dekel06} $K$, the function $l(\w)$ in (\ref{eq:l}) has a positive RSC and finite RSM parameters. 
\end{lemma}

\begin{theorem}[Weak submodularity]
\label{ws}
The set function $f$ in (\ref{def:f}) is weakly submodular with the submodularity ratio $\gamma > 0$.
\end{theorem}

\begin{remark}
Lemma \ref{mono} and Theorem \ref{ws} imply that algorithm \ref{protogreedy}, ProtoGreedy, has an approximation guarantee of $\left(1-e^{-\gamma}\right)$ \cite{Nemhauser78}.
\end{remark}

\begin{lemma}
\label{lem:KKTGradient}
For $j \notin L$, if $\gradljbzetaL \leq 0$ then $\bzetaLj = \bzetaL$. In particular $\bzetaLj_j = \bzetaL_j=0$. Hence, if $\bzetaLj_j >0$ then $\gradljbzetaL > 0$.
\end{lemma}

\begin{theorem}[ProtoDash Guarantees]
\label{pd}
If $L^{D}$ is the $m$ sparse set selected by ProtoDash and $L^{\ast}$ is the optimal $m$ sparse set then,
\begin{equation}
f\left(L^D\right) \ge f\left(L^{\ast}\right) \left[1- e^{- \frac{c_{2m}}{\tilde{C}_1}} \right].
\end{equation}
\end{theorem}

\begin{corollary}
\label{cor1}
In ProtoDash, at each iteration, selecting the next prototype with the maximum gradient is equivalent to choosing a prototype that maximizes a tight lower bound on the function maximized by ProtoGreedy for its selection of the next prototype.
\end{corollary}

\section{Time Complexity}
\label{sec:timecomplexity}
For both ProtoGreedy and ProtoDash we need to compute the mean inner product of $X^{(1)}$ for instances in $X^{(2)}$, which takes $O(n^{(1)}n^{(2)})$ time. The time complexity to compute inner products between points in data set $X^{(2)}$ to build the kernel matrix $K$ is $O(mn^{(2)})$.

For ProtoGreedy, the selection of the next best element requires running $O(n^{(2)})$ quadratic programs each taking $O(m^3)$. Hence the time required for choosing $m$ such next best elements is $O(m^4n^{(2)})$. The total time complexity of ProtoGreedy is $O\left(n^{(2)}(n^{(1)}+m^4)\right)$.
The $i^{th}$ iteration of ProtoDash requires a search over $(n^{(2)} - i +1)$ elements to determine the next best element. Computing gradient for each element searched is $O(i)$ as it involves computing inner products with the chosen $i-1$ prototypes. Hence the complexity of choosing $m$ such next best elements is $O(m^2n^{(2)})$. For each iteration $i$, we need to run one quadratic program needing $O(i^3)$ time to compute weights. Hence, overall it is $O(m^4)$. Consequently, the total time complexity for ProtoDash is $O\left(n^{(2)}(n^{(1)} + m^2) + m^4\right)$.

Given our motivation of interpretability which requires concise data summarization where typically $m<<n^{(2)}$, ProtoDash will be significantly faster than ProtoGreedy as witnessed in our experiments.

\begin{figure*}[t]
  \begin{center}
      \includegraphics[width=0.45\linewidth]{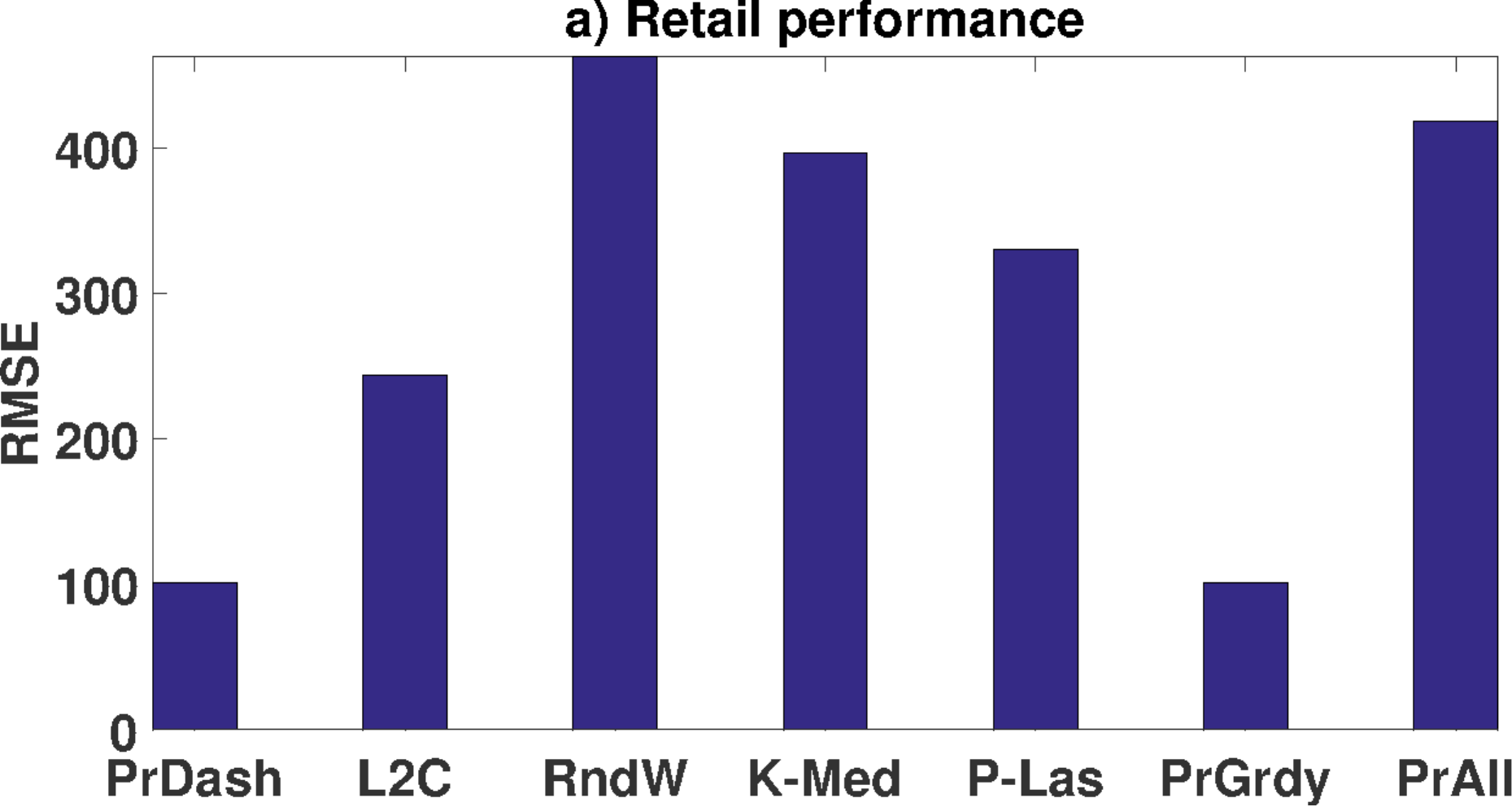}
      \includegraphics[width=0.45\linewidth]{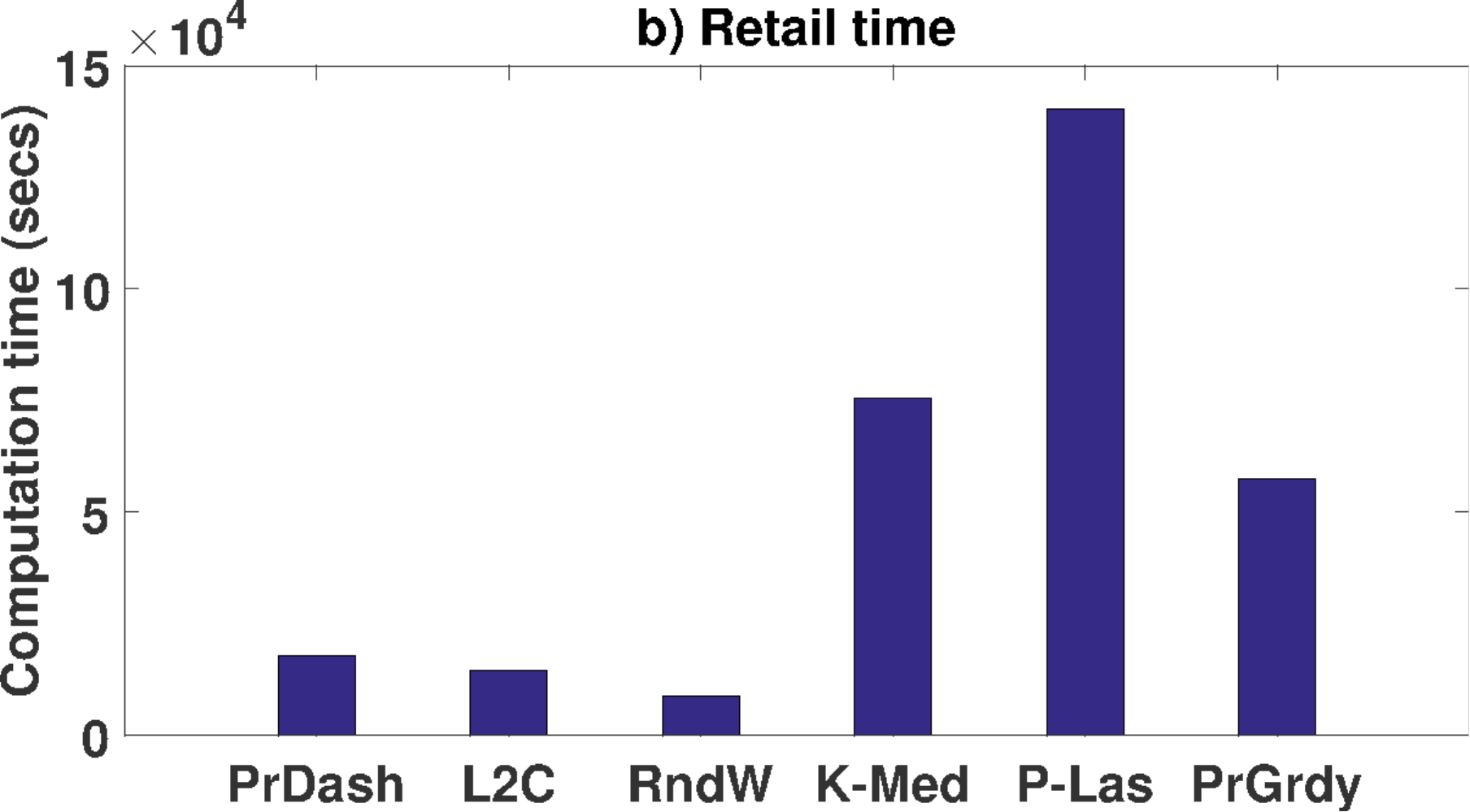} 
  \end{center}
  \caption{We observe the quantitative results of the different methods on retail dataset.}
  \label{retail}
\end{figure*}

\begin{figure}[t]
  \begin{center}
    \begin{tabular}{c}
      \includegraphics[width=0.45\textwidth, height=0.2\textheight]{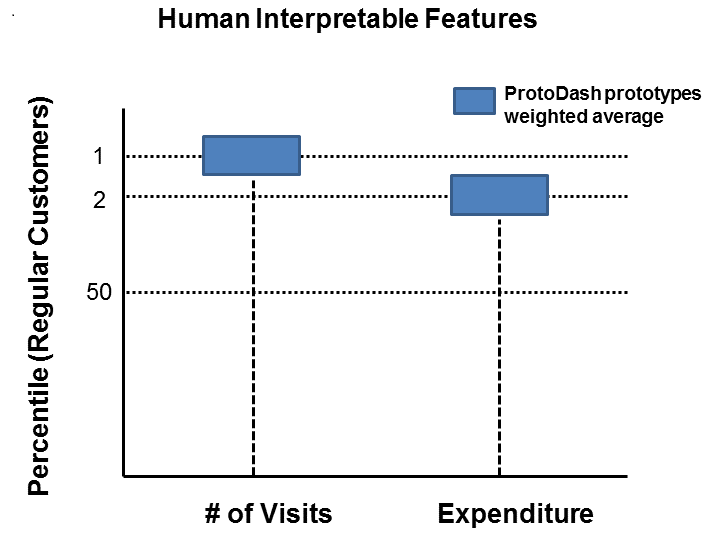}
     \end{tabular}
  \end{center}
  \caption{We see that the weighted average of our prototypes selected from the regular customers group that best fit the loyalty group are in the top 1 and top 2 percentile respectively for two important interpretable features, \# of visits and expenditure.}
  \label{retailQF}
\end{figure}

\section{Experiments}
In this section we quantitatively as well as qualitatively validate our algorithms on three diverse domains. The first is a dataset from a large retailer. The second is MNIST which is a handwritten digit dataset. The third are 40 health questionnaires obtained from the CDC website.

We compare ProtoDash (or PrDash) with five other methods. The first is our slower but potentially better performing greedy method ProtoGreedy (or PrGrdy). The second is L2C. The third is P-Lasso (or P-Las), i.e., lasso with the non-negativity constraint \cite{classo}. The P-Lasso objective in our setting is the equivalent Ivanov formulation of standard Lasso \cite{Ivanov76} where we maximize $l(\w)\text{ subject to }||\w||_1 \le \epsilon, \w \ge0$
where for every choice of $\lambda$ in Lasso, there is an equivalent $\epsilon$ in the Ivanov formulation \cite{Oneto16} \footnote{For P-Lasso we try to find the regularization parameter $\epsilon$ that gives roughly (within 10\%) the same number of prototypes as the other methods. This was not always possible.}. The running time of P-Lasso is $O(n^{(1)}n^{(2)} + {n^{(2)}}^3)$, making it much more expensive than ProtoDash as we will see in the experiments. The fourth is K-Medoids (or K-Med) \cite{sproto}. The fifth is RandomW (or RndW), where prototypes are selected randomly, but the weights are computed based on our strategy.\eat{We use abridged versions of their names in the figures for readability. The association is obvious but to be clear ProtoDash is PrDash, RandomW is RndW, K-Medoids is K-Med, P-Lasso is P-Las and ProtoGreedy is PrGrdy.} ProtoGreedy's and ProtoDash's superior performance to this baseline as well as to L2C implies that selecting high quality prototypes in conjunction with determining their weights are important for obtaining state-of-the-art results and that neither of these strategies suffices in isolation. We use Gaussian kernel in all the experiments. The kernel width is chosen by cross-validation. 

Additional experimental results on MNIST where we initially oversample to pick $2m$ and $3m$ prototypes and among them choose the top $m$ based on the magnitude of the learned weights are given in Figure \ref{mnistinit}c. This is another benefit of learning the weights where we can first oversample and then choose the desired number of prototypes from this set that have the largest weights. This often leads to superior results.

\subsection{Retail}
The first dataset we consider is a proprietary E-commerce dataset from a large retailer. We have 2 years of online customer data from the beginning of 2015 to the end of 2016. This is information of roughly 80 million customers. Around 2 million of which are loyalty customers and we know of 9878 customers who were regular customers in 2015 but became loyalty in 2016.

The goal is to accurately predict the total expenditure of a customer and to evaluate if being a loyalty or a regular customer has any effect on his behavior independent of factors such as the number of online visits, his geo or zip, average time per visit, average number of pages viewed per visit, brand affinities, color and finish affinities, which are the attributes in the dataset.

To answer this counter-factual we build a SVM-RBF \cite{libsvm} model using 10-fold cross-validation on the 2016 data and evaluate its performance on the 2015 data for the 9878 customers that were among the loyalty group in 2016 but not in 2015. In essence we test our model by evaluating how accurately we predict the expenditure of these 9878 customers in 2015, with a model that is built using the 2016 data.

The 2016 data that we use to train the SVM-RBF depends on the prototype selection methods. The entire loyalty group is always part of the training. The question is identify the subset of the regular customers we should also add to training. For our methods we choose prototypes from the regular customer base that best represent the loyalty group. We select around 1.5 million customers because the improvement in objective is incremental beyond this point. We select the same number of prototypes for the competing methods. For this experiment we have an additional baseline which is training using all the data referred to as PrAll. 

\noindent\textbf{Quantitative Evaluation:}
In Figure \ref{retail}a, we observe the root mean squared errors (RMSE) of the different methods. We see that our methods are significantly better than the competitors. Using all the data is not a good idea probably due to the high size imbalance between the two groups. We also observe that ProtoDash is almost as good as ProtoGreedy. In Figure \ref{retail}b, we note the running time of the different prototype selection methods. Here we see that ProtoDash is close in running time to L2C and over 3 times faster than ProtoGreedy. Hence, from Figures \ref{retail}a and \ref{retail}b we can conclude that ProtoDash would be the method of choice for this application.

\noindent\textbf{Qualitative Evaluation:}
We did further investigation of our prototypes w.r.t. features that experts consider important. We found that our prototype group had high number of visits i.e. based on our weights the (weighted) average number of visits for this group was in the top 1\% of the visits by regular customers and they also had relatively high expenditure, i.e. the weighted average was in the top 2\% in this group. This is shown in Figure \ref{retailQF}.

\emph{The even more reassuring fact was that when we shared with the domain experts the top 100 prototypes based on our weights of the 1.5 million regular customers that were selected as prototypical of the loyalty group in 2016, they confirmed that 83 of those 100 became loyalty customers in 2017.}

\begin{figure*}[t]
  \begin{center}
    \includegraphics[width=0.32\linewidth, height=2.9cm]{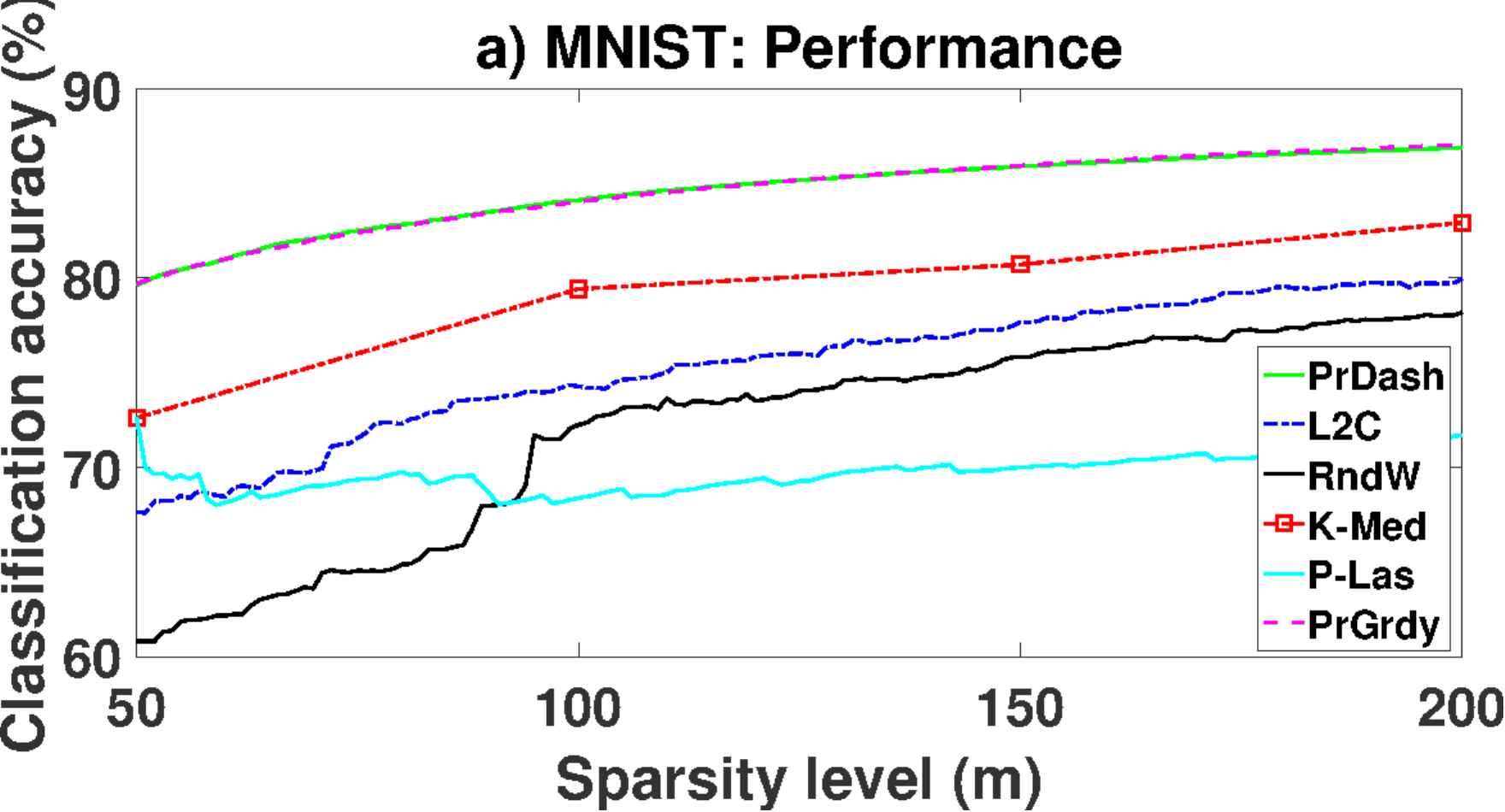}
      \includegraphics[width=0.32\linewidth]{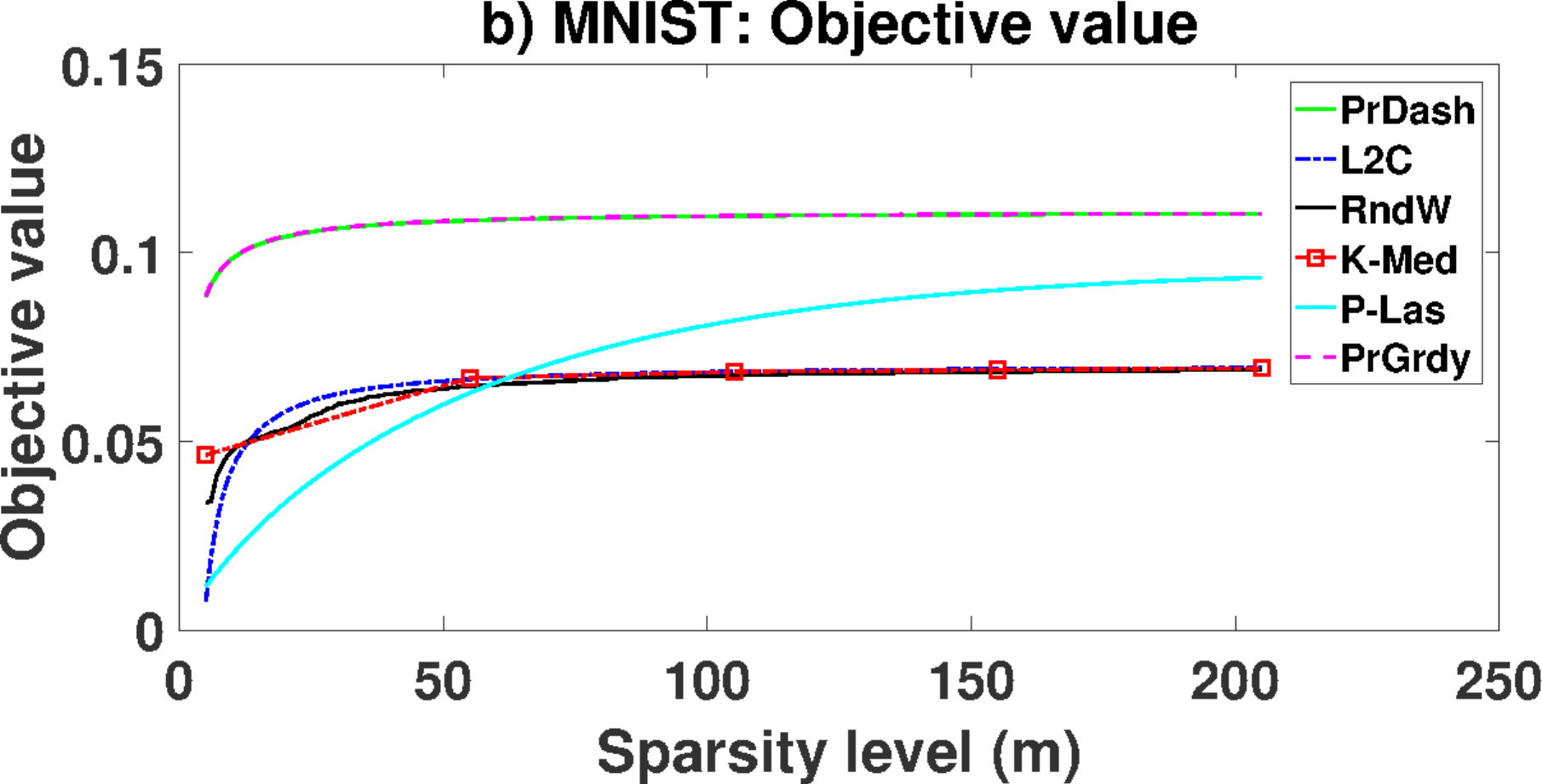}
      \includegraphics[width=0.32\linewidth]{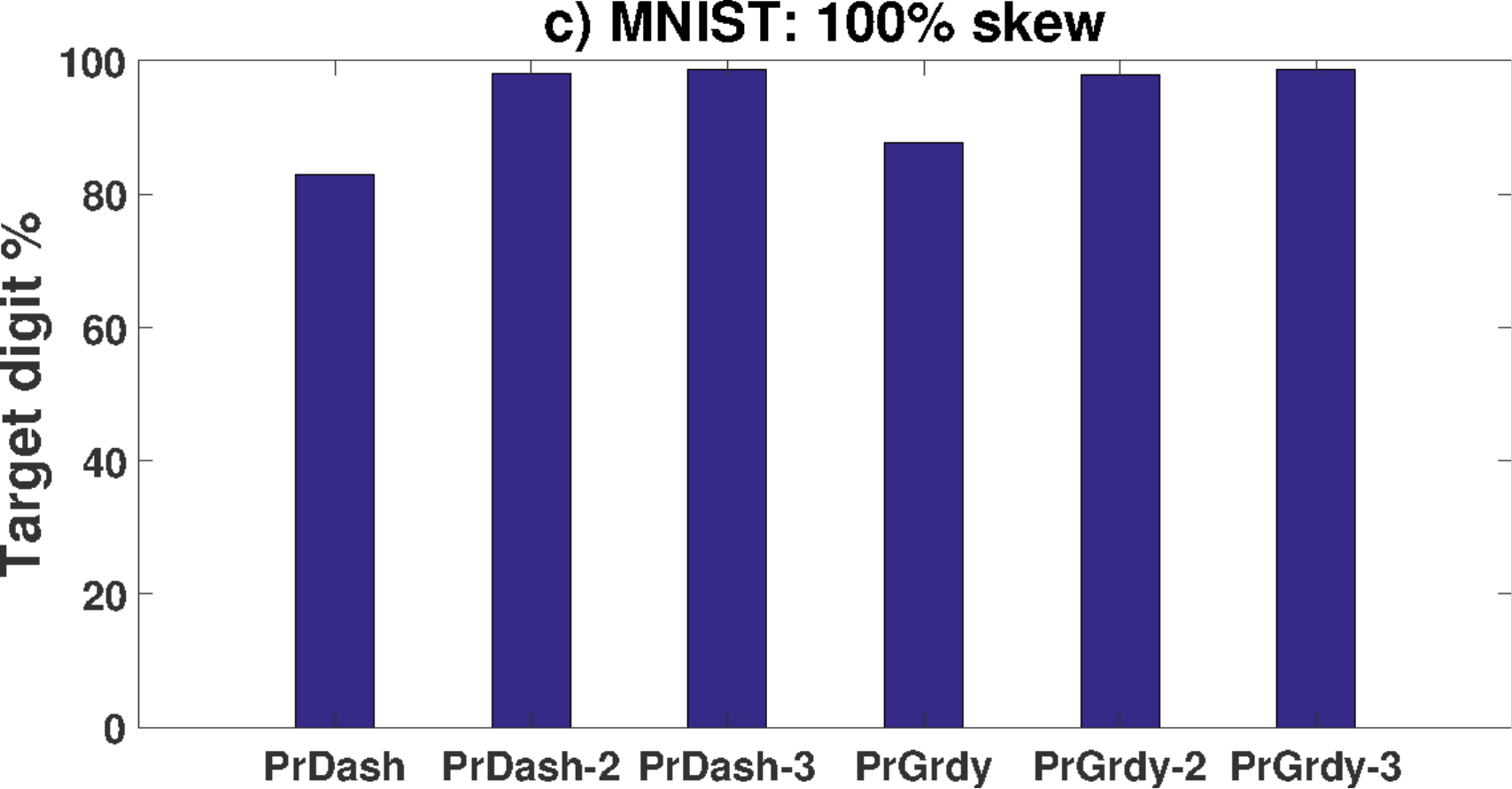}       
  \end{center}
  \caption{In a) we observe the (averaged) source dataset classification accuracy values for different number of prototypes. In b) we observe the (averaged) $l(.)$ value for different number of prototypes. In c) we observe the \% of the target digit selected from the source dataset by our methods based on our learned (highest) weights at 100\% target skew for different oversampling ratios (i.e. choosing highest weighted $m$ prototypes from $2m$ and $3m$ selected prototypes).}
  \label{mnistinit}
\end{figure*}

\begin{figure*}[t]
  \begin{center}
      \includegraphics[width=0.32\linewidth]{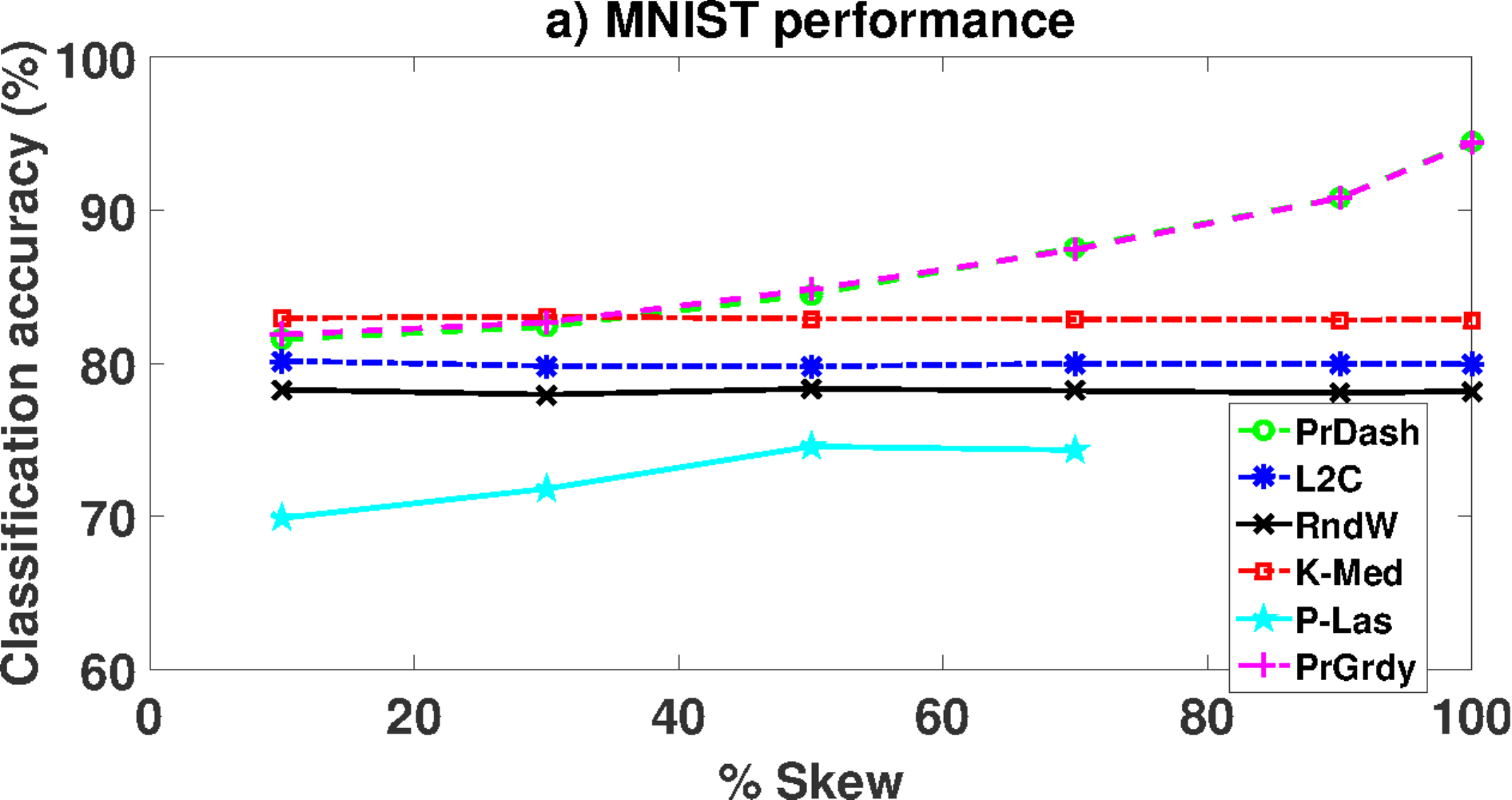}      \includegraphics[width=0.32\linewidth]{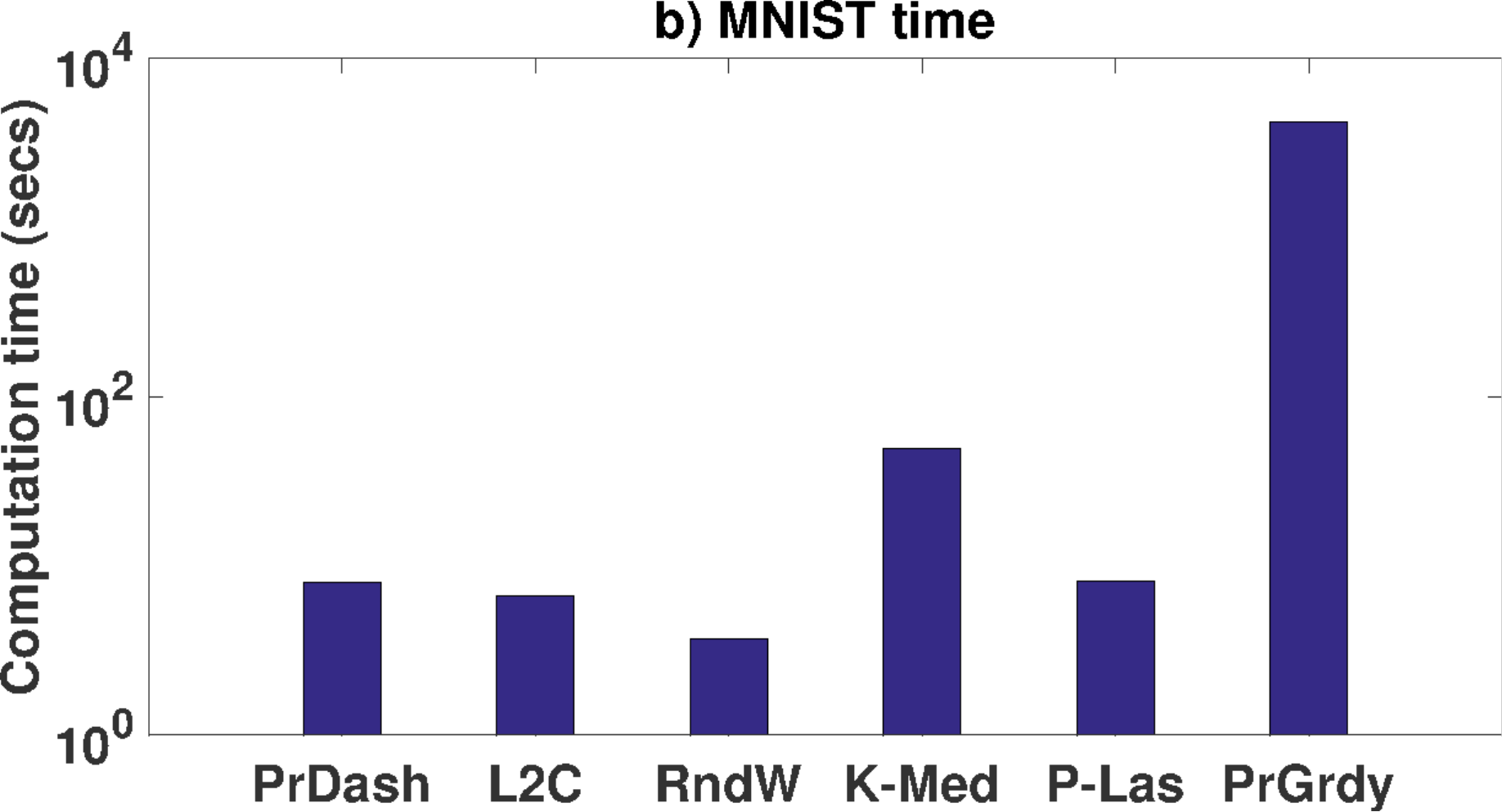} 
      \includegraphics[width=0.32\linewidth]{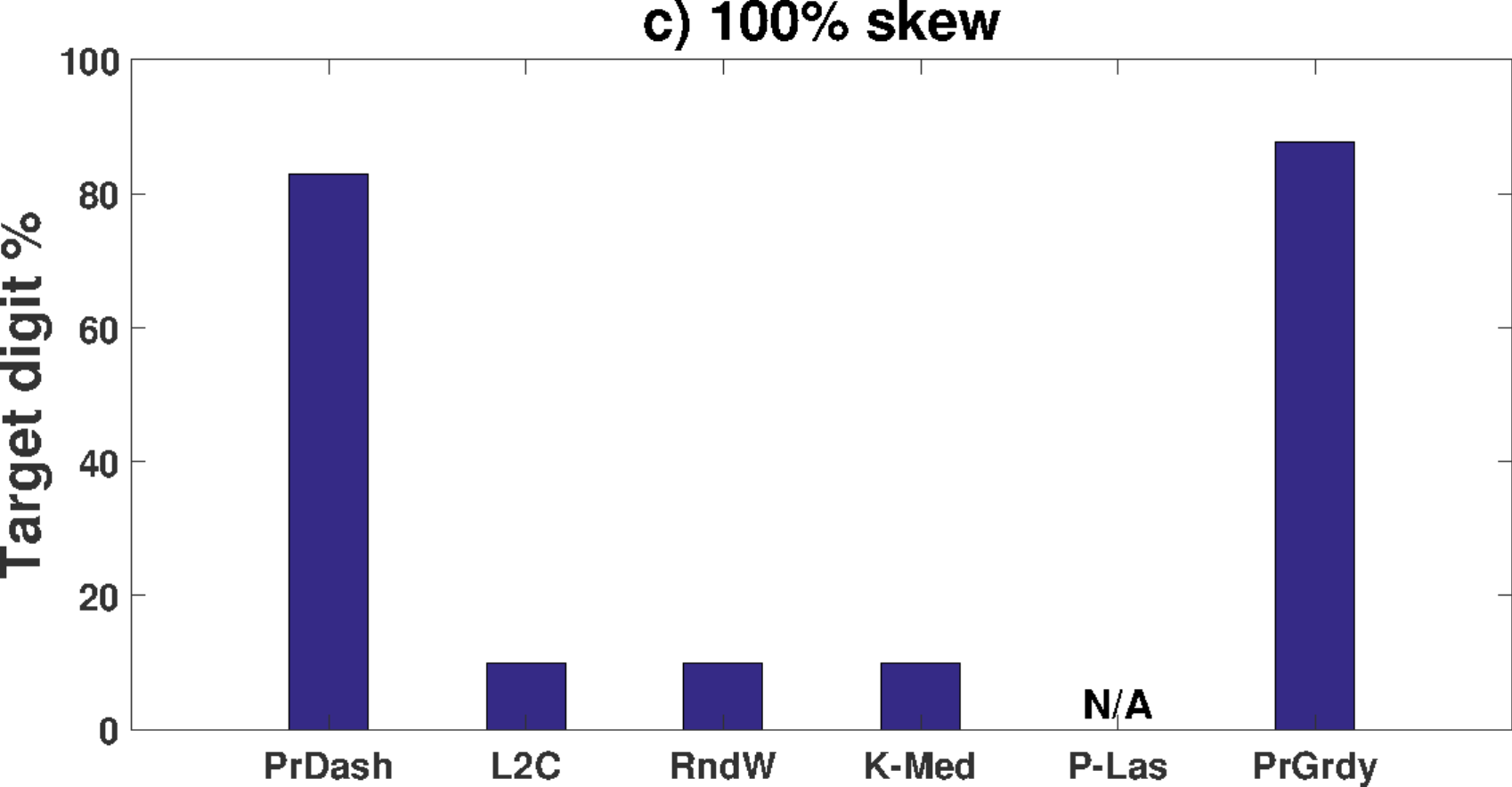}
  \end{center}
  \caption{We observe the quantitative results of the different methods on MNIST.}
  \label{mnist}
\end{figure*}

\begin{figure*}[t]
  \begin{center}
    \includegraphics[width=0.32\linewidth,height=2.9cm]{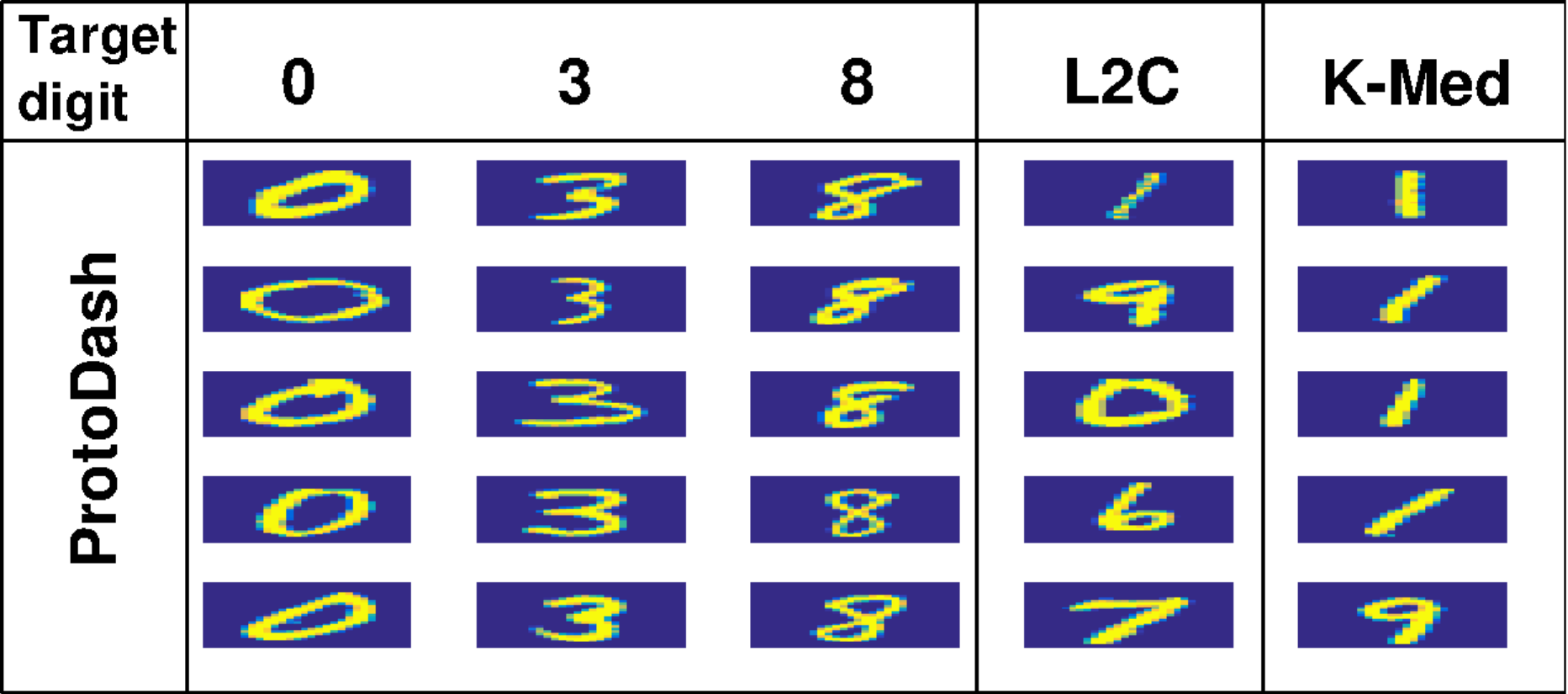}
      \includegraphics[width=0.32\linewidth]{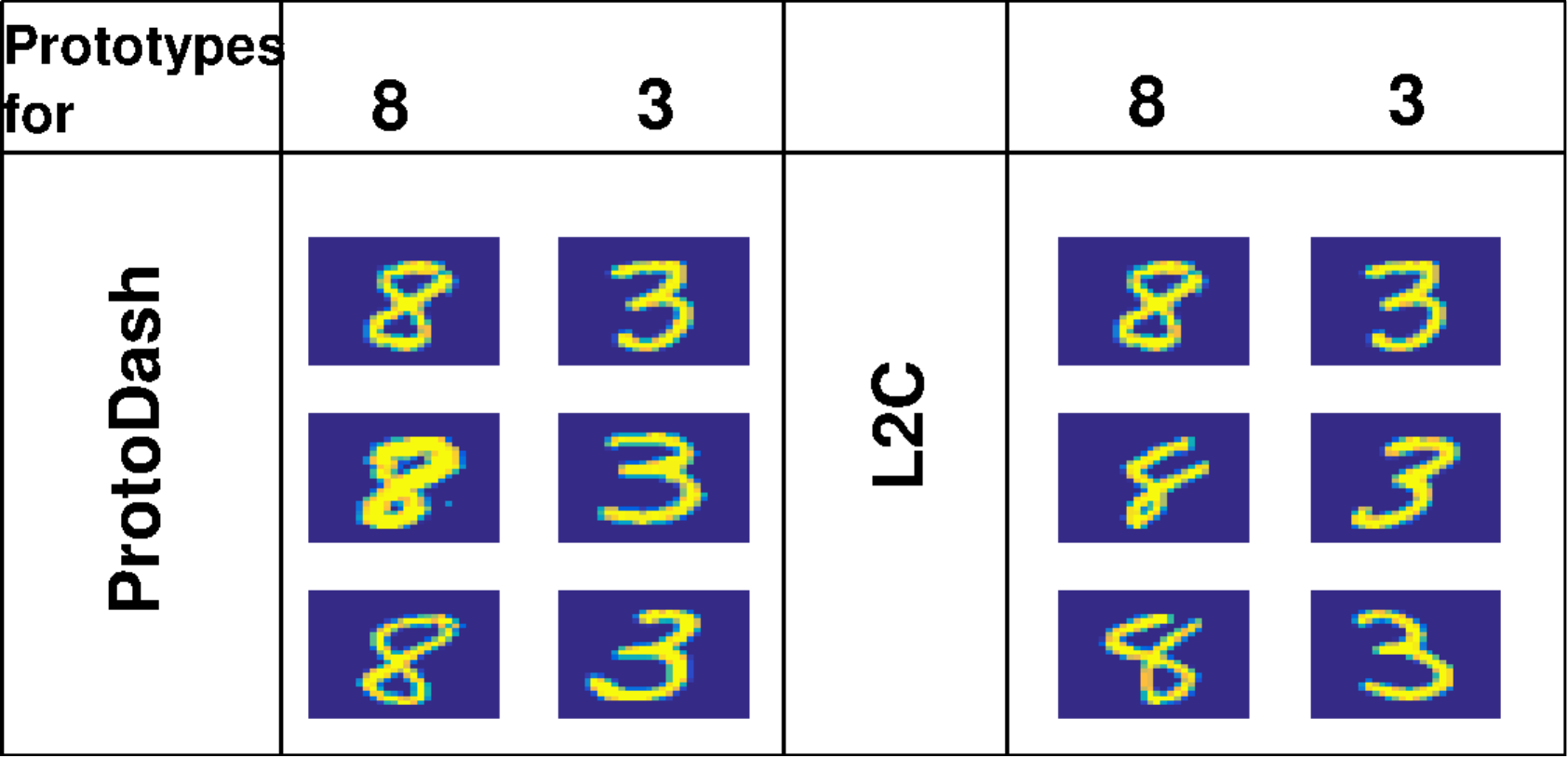} 
      \includegraphics[width=0.32\linewidth]{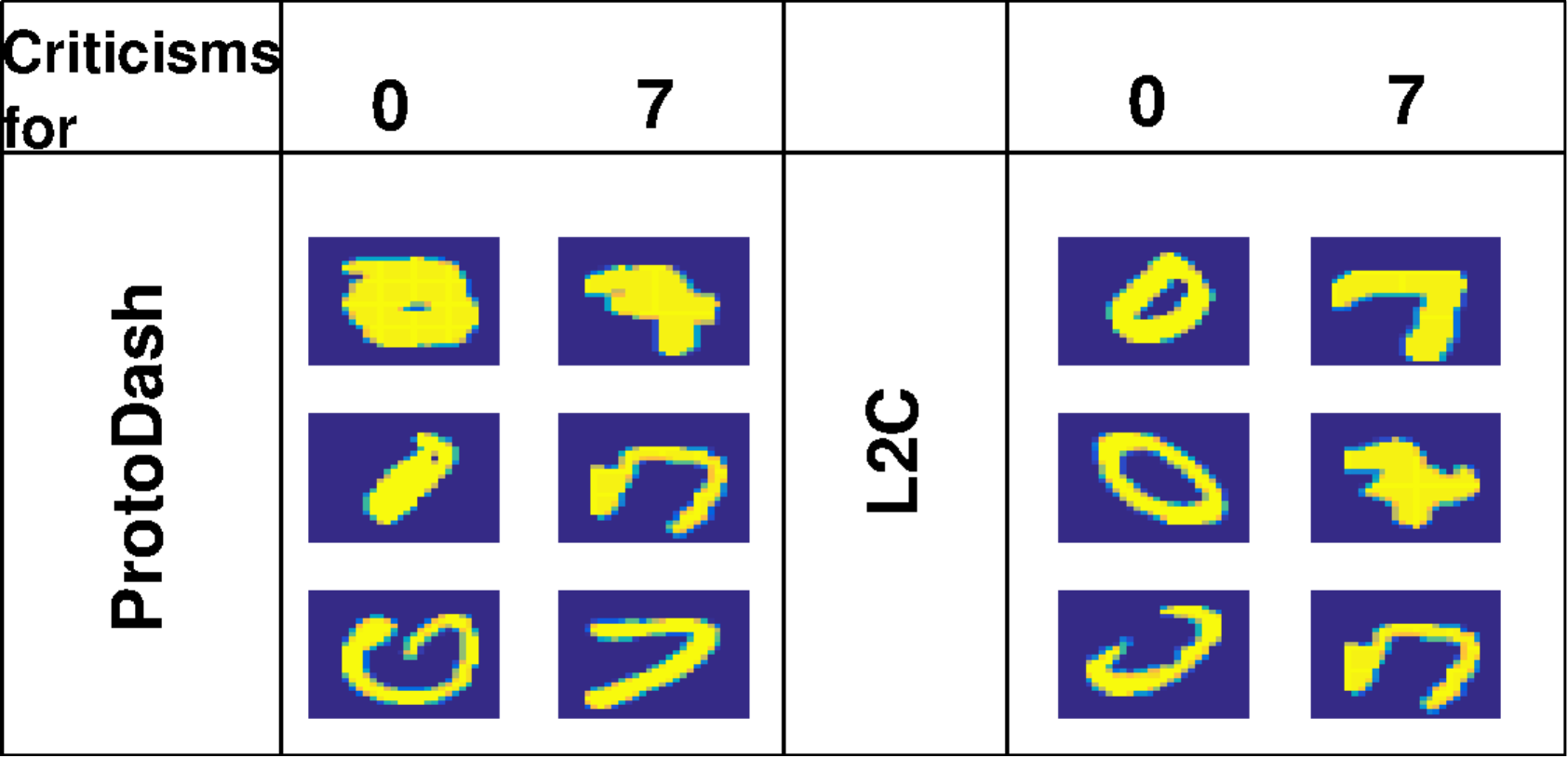}       
  \end{center}
  \caption{We observe the qualitative (or visually discernible) results for ProtoDash and L2C above. The ordering from best to worst candidates for the specific task is top to down. Our ordering is determined by our learned weights.}
  \label{mnistql}
\end{figure*}

\subsection{MNIST}
On MNIST dataset we design experiments to validate the adaptability of our approach in terms of how well we are able to select prototypes (and criticisms) from a source dataset ($X^{(2)}$) that best match the distribution of a different target dataset ($X^{(1)}$). This has potential applications in transfer learning, multitask learning and covariate shift correction.

We employ the (global) one nearest neighbor (1-NN) prototype classifier as described in \cite{Kim16}. For our methods, since our learned weights and the distance metric in 1-NN classification are not on the same scale, we used the weights to select the top $m$ prototypes and then based on these performed 1-NN classification. To obtain more robust results especially as we skew the target distribution towards a single digit, we use the MNIST training set to form multiple target datasets, while the MNIST test set is used as the source dataset. In particular, we use the MNIST training set of 60000 images to form multiple target sets of size 5420, which is the cardinality of least frequent digit in this set. We then randomly select 1500 images from the original MNIST test set to form our source dataset. The first target set we form is representative of the population and contains an equal number (i.e. $\sim$ 10\%) of all the digits. We now create skewed target sets for percentages of $s=30, 50, 70, 90$ and $100$. For each value of $s$ we create 10 target sets where a particular digit is $s$ fraction of the target set and the remaining portion of target set contains representative population of the other digits. For example, when $s=70$ one of the test sets will have 70\% 0s and the remaining 30\% is shared equally by the other 9 digits. By averaging our results for each $s$ we can observe the performance of the different methods for varying levels of skew. The reported results are over 100 such re-sampling trials.
\begin{figure*}[htbp]
  \begin{center}
      \includegraphics[width=0.32\linewidth]{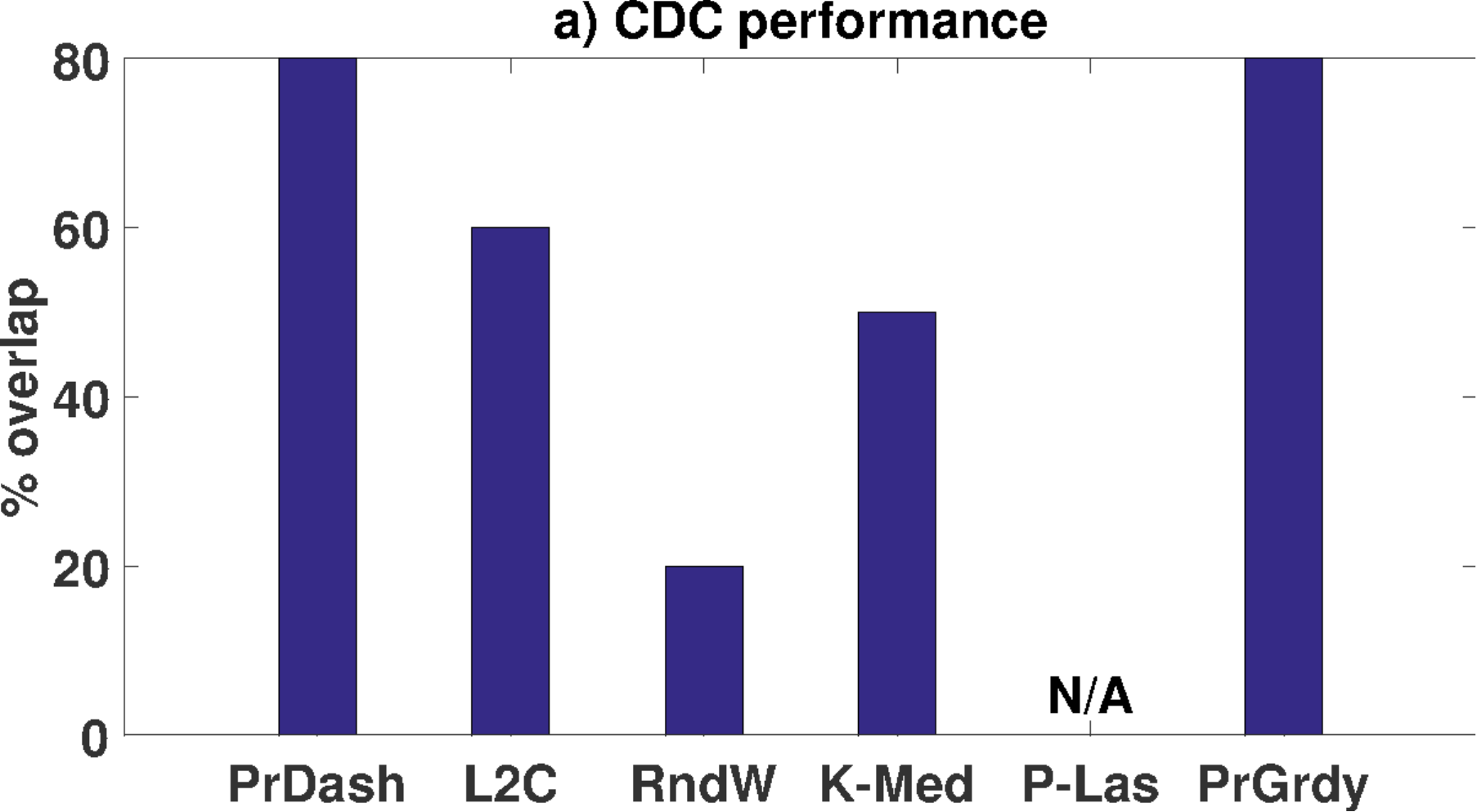} 
      \includegraphics[width=0.32\linewidth]{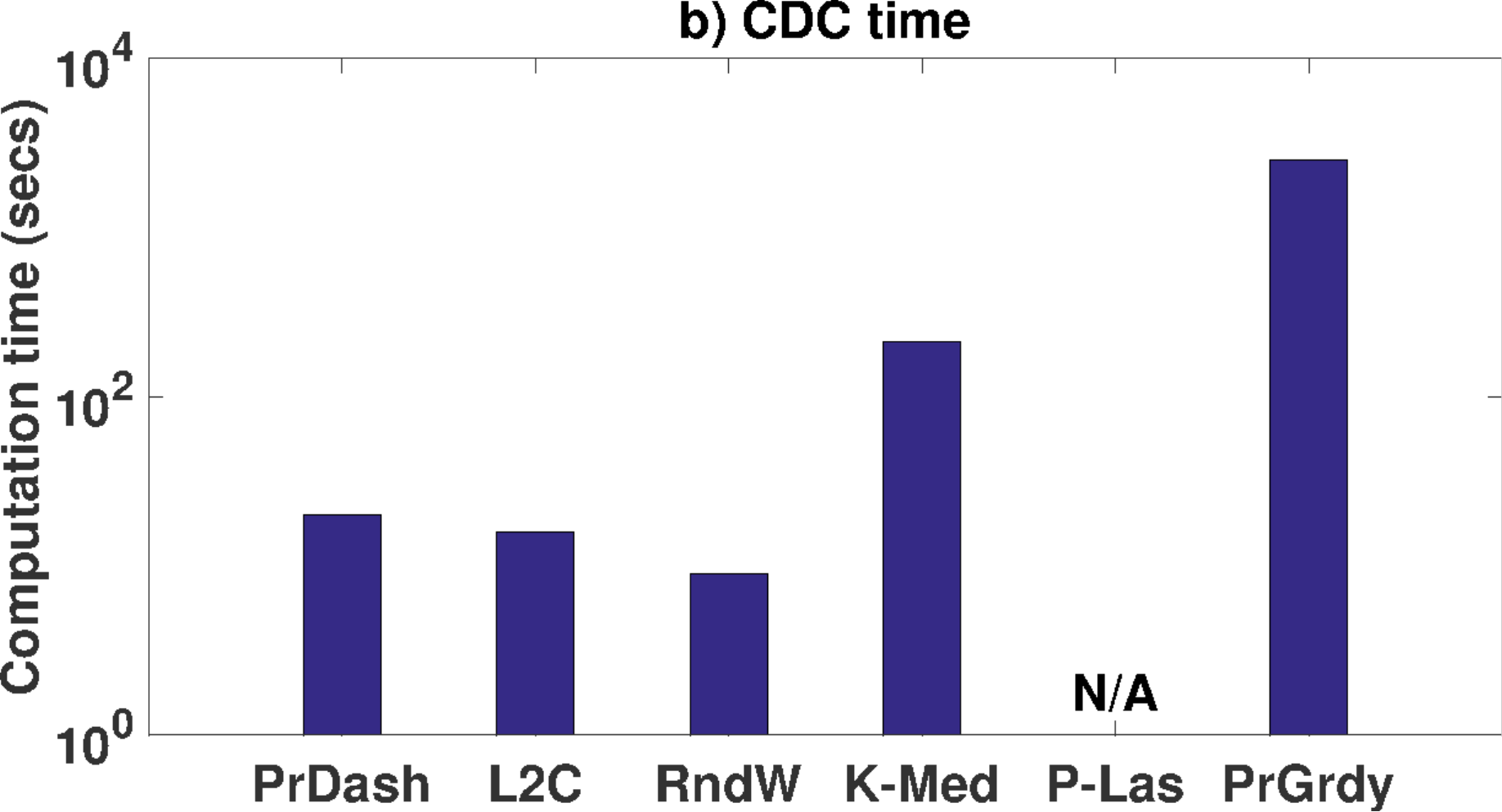} 
      \includegraphics[width=0.34\linewidth]{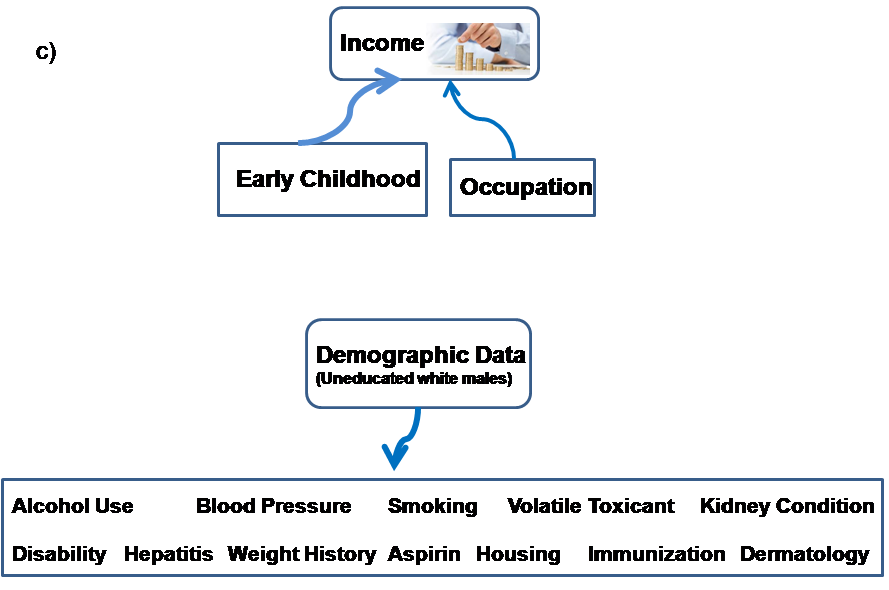}
  \end{center}
  \caption{Above are the quantitative and qualitative results on the CDC Questionnaires.}
  \label{cdc}
\end{figure*}

\noindent\textbf{Quantitative Evaluation:}
If we average over all percentages $s$ and plot the classification accuracy values as well as our objective for different levels of sparsity as seen in Figures \ref{mnistinit}a and \ref{mnistinit}b, we find that around $m=200$, the gain in objective is incremental. We thus choose 200 prototypes for all the methods. In Figure \ref{mnist}a, we see that the performance of the closest competitors on the target set does not adapt to skew. Our methods are a little worse than K-Medoids initially but their performance drastically improves as the skew increases. \emph{This scenario shows the true power of our methods in being able to adapt to non-representative target distributions that are significantly different than the source.} Additionally, the performance of ProtoDash again is indistinguishable from ProtoGreedy. In Figure \ref{mnist}b, we see that ProtoDash is orders of magnitude faster than K-Medoids and ProtoGreedy which are its closest competitors.

To further understand why our methods so significantly outperform the others at high skews of the target set, in Figure \ref{mnist}c, we report the percentages of the target digit picked by the different methods from the source set averaged over the 10 digits at 100\% target skew (i.e. target dataset contains copies of only a single digit). We see that our methods adapt swiftly by picking almost exclusively images of the target digit from the source dataset with all the weight concentrated on them. The weighting is further justified when we look at Figure \ref{mnistinit}c where we oversample the prototypes and then choose the desired number $m$ from this based on highest weights. PrDash and PrGrdy are results from just choosing $m$
prototypes. However, PrDash-2 and PrDash-3 are results when we first select $2m$ and $3m$ prototypes and among them pick the highest weighted $m$ prototypes. We underscore here that by oversampling, we are able to select more prototypes that match the target digit. The same applies to PrGrdy. This showcases another benefit of learning non-negative weights where we can first oversample and then use the weights to clearly identify important, top prototypes leading to even better results.

\noindent\textbf{Visual Evaluation:}
When we have 100\% (target) skew, we see in Figure \ref{mnistql} (left) that our method exclusively picks the target digit from the source dataset such as 0s or 3s or 8s, while L2C and K-Medoids pick a set that is independent of the target digit. This is a visual illustration of the result reported in Figure \ref{mnist}c. We also wanted to measure the quality of the prototypes and criticisms for a specific digit in the data set without trying to fit any target distribution, i.e. $X^{(1)}=X^{(2)}=$ samples of a specific digit in the data. In Figure \ref{mnistql} (center) we observe the top three prototypes for digits 8 and 3 for ProtoDash and L2C. Our prototypes, which were \emph{selected based on weights}, look visually more appealing where the 8s for instance are \emph{complete with no broken curves} which is not the case for the second and third best prototype of L2C. For criticisms too in Figure \ref{mnistql} (right), which are the farthest away examples from the prototypes, we see that our criticisms seem to be better. For instance, our 0s look visually much worse than those selected by L2C. Also our ordering of "bad" 7s seems to be much better.

\subsection{CDC Questionnaires Case Study}
The US dept. of health conducts surveys consisting of 10s of questionnaires sent to over thousands of people every couple of years. This is a rich repository of anonymized human health facts that are publicly available. We in this study use the health questionnaires collected over the 2013-2014 period \cite{Nhanes}. There are 43 questionnaires (viz. \emph{Alcohol Use, Occupation, Income, Early Childhood, Depression, Diet}) of which we use 40 (3 had data issues).\eat{ and which are listed in the appendix.} Note that the questionnaires are answered by the same set of 5924 people, so in essence, we can view this setup as having 40 datasets over the same examples but with different features where the features correspond to questions in the questionnaires.

An expert in public health wanted to see: task 1) if we could rank order the questionnaires based on some measure of importance so that henceforth they could potentially send fewer questionnaires for people to fill, and task 2) if for a given questionnaire we could find amid the other 39 questionnaires the one that is most representative of it. Such insight could lead to early interventions that can potentially save lives.

We attacked both problems with our prototype selection framework. In fact, accomplishing the 2$^{nd}$ task is a big step in resolving the 1$^{st}$. For each questionnaire $Q_i\in Q$ where $Q=\{Q_1,...,Q_{40}\}$ we found prototypes ($m=10$) after which the improvement in objective (eq.~\ref{def:f}) was incremental. We then evaluated the quality of these prototypes on the other 39 questionnaires based on the same objective in (\ref{def:f}). Thus, for a particular $Q_i$ we rank ordered the other $Q/Q_i$ based on our objective value. Ergo, the rank $r_{ij}$ signifies how well the prototypes of a questionnaire $Q_j$ represents $Q_i$. This resolves the task no. 2 above. Note that the rank is not commutative and hence graphically it can be viewed as a directed graph. To answer task no. 1, for each $Q_j$ we found its average rank i.e. $r_j=\frac{1}{39}\sum_{i\in\{1,...,40\},i\neq j}r_{ij}$ across other questionnaires and sorted the $r_j$s in ascending order. Thus, lower the $r_j$ more important the questionnaire.

\noindent\textbf{Human Expert based Evaluation:}
We obtained from the expert a list of 10 questionnaires he thought would be most important of the 40. We intersected this list with the top 10 by the different methods and report the overlap percentage in Figure \ref{cdc}a. P-Lasso didn't produce results possibly due to bad condition number on most datasets so we omitted it. We see that our methods have the largest overlap and thus have the most agreement with the expert. We also again see in Figure \ref{cdc}b the efficiency of ProtoDash.

\noindent\textbf{Insights Matching Scientific Studies:}
We tried to validate some of the rankings we got from task no. 2 based on prior studies. The insights are depicted in Figure \ref{cdc}c. We found that for the Income questionnaire its best representative prototypes, came from the Early Childhood questionnaire which has information about the environment in which the child was born. The second best questionnaire was Occupation. Occupation is intuitive to understand as affecting income. However, Early childhood is interesting and the expert mentioned that there is validation of this based on a recent study which talks about significant decrease in social mobility in recent years \cite{poor20}. The ranking of other methods differed with no such justification.

We also analyzed the Demographic data questionnaire from the same year in terms of how it fared with representing the 40 questionnaires. It turned out that it was the top ranked for multiple questionnaires as shown in Figure \ref{cdc}c, which are indicative of high stress levels due to health issues or financial condition.\eat{ Alcohol Use, Blood Pressure \& Cholesterol, Smoking - Household Smokers, Smoking - Secondhand Smoke Exposure, Volatile Toxicant, Weight History - Youth, Disability, Preventive Aspirin Use, Kidney Conditions - Urology, Hepatitis, Housing Characteristics, Immunization and  Dermatology.} Some of the most common highest weighted prototypes were white Americans with education levels that were at best AA. This is highly consistent with the recent study \cite{dod} which shows that the death toll among middle aged uneducated white Americans is on the rise due to financial and health related stresses.
\vspace{-2mm}
\section{Discussion}
In this paper we provided a fast prototype selection method ProtoDash. We derived approximation guarantees for it and showed in the experiments that its performance is as good as the standard greedy version ProtoGreedy but computationally much faster. Learning non-negative weights and its ability to find prototypes across datasets leads to its superior performance over L2C and other competitors, while still outputting results that are insightful and easy to evaluate. 

In the future, it would be interesting to further close the theoretical gap between ProtoDash and ProtoGreedy maybe based on the ideas in \cite{dg}, although it is not clear if they would generalize to our setting. The other extension may be to obtain a convex combination of weights in addition to non-negativity. In terms of practical applications, we are in the process of further studying how demographics and other behavioral traits relate to statistics on increased mortality rates \cite{dod}, which has been a major concern in the recent decades. Our prototype selection methods could also have applications in transfer learning and lifelong learning applications, where one can use the prototypes to efficiently and accurately learn models for new tasks. We plan to explore such avenues in the future.
\appendix
\subsection{Proof of Lemma~\ref{mono}}
Let $|L_1| = n_1$ and $|L_2|=n_2$ and $n_1 \leq n_2$. Index the elements in $L_2$ such that the first $n_1$ elements are those contained in $L_1$.
Then,
\begin{equation*}
\begin{split}
f\left(L_2\right) &= \max\limits_{\w: supp(\w) \in L_2, w_j \geq 0} l\left(\w\right) 
      \geq \max\limits_{\w: supp(\w) \in L_1, w_j \geq 0} l\left(\w\right)\\&= f\left(L_1\right).
      \end{split}
\end{equation*}

\subsection{Proof of Lemma~\ref{lemma:RSCRSM}}
For the concave function $l(\w) = -\frac{1}{2} \w^TK\w + \w^T\bmu_p$, we calculate $l(\w_1)-l(\w_2) - \nabla  \langle l(\w_2), \w_1-\w_2\rangle = -0.5(\w_1-\w_2)^TK(\w_1-\w_2)$. If $\w_1$ and $\w_2$ are $k_1$ and $k_2$ sparse vectors respectively, then $\Delta \w = \w_1-\w_2$ has a \emph{maximum} of $k \leq k_1+k_2$ non-zero entries. For the constants $c$ and $C$ satisfying $-c\|\Delta \w\|^2 \ge -\Delta\w^TK\Delta \w \ge -C \|\Delta \w\|^2$ we obtain the bounds: $c\ge k$-sparse smallest eigenvalue of $K$ and $C\le k$-sparse largest eigenvalue of $K$. In particular, when $supp(w_2) \subset supp(w_1)$, $\left\|\Delta \w\right\|_0 \leq k_1$ providing tighter bounds for $c$ and $C$.

\subsection{Proof of Theorem~\ref{ws}}
We lower bound the numerator and upper bound the denominator. Let $\bar{m} = |L|+|S|$. Recall that $\bzetaL,    \bzeta^{(L \cup S)} \in \R^{b^+}$ are the maximizer $\lbzetaL=f(L)$ and $\lbzetaLS= f(L\cup S)$ respectively. By the definition of $RSC$ and $RSM$ constants we find
\begin{align*}
\frac{c_{\bar{m}}}{2} \left\|\bzetaLS-\bzetaL\right\|^2 \leq &\lbzetaL - \lbzetaLS \\
& + \left\langle\nabla \lbzetaL, \bzetaLS-\bzetaL \right\rangle.
\end{align*}
Noting that $f$ is monotone for increasing supports we get
\begin{align}
&\lbzetaLS-\lbzetaL \\ &\leq \left\langle\nabla \lbzetaL, \bzetaLS-\bzetaL \right\rangle- \frac{c_{\bar{m}}}{2} \left\|\bzetaLS-\bzetaL\right\|^2 \nonumber \\
\label{eq:lowerbound}
& \leq \max\limits_{\bv: \bv_{(L\cup S)^c}=0, \bv >=0} \left\langle\nabla \lbzetaL, \bv-\bzetaL \right\rangle - \frac{c_{\bar{m}}}{2} \left\|\bv-\bzetaL\right\|^2.
\end{align}  
The vector $\bv$ with the support restricted to the coordinates specified by $L \cup S$ attains maximum at
\begin{equation*}
\bv_{L\cup S} = \max \left\{ \frac{1}{c_{\bar{m}}} \nabla l_{L \cup S}\left(\bzetaL \right) + \bzetaL_{L \cup S}, \mathbf{0} \right \}.
\end{equation*}
It then follows
\begin{equation*}
(\bv-\bzetaL)_{L\cup S} = \max \left\{ \frac{1}{c_{\bar{m}}} \nabla l_{L \cup S}\left(\bzetaL \right), - \bzetaL_{L \cup S} \right \}.
\end{equation*}

The KKT conditions at the optimum $\bzetaL$ for the function $f(L)$ necessitates that $\forall j \in L$,
\begin{align*}
\bzetaL_j > 0 &\implies \gradljbzetaL = 0, \\
\bzetaL_j = 0 &\implies \gradljbzetaL \leq 0
\end{align*}
and hence we have $(\bv-\bzetaL)_{j} = 0$. Further, for $j \in S$, $\bzetaL_j=0$ implying that $(\bv-\bzetaL)_{j} =  \max \left\{ \frac{1}{c_{\bar{m}}} \gradljbzetaL, 0 \right \}$. Defining $\nabla l_{S}^+\left(\bzetaL \right) = \max \left\{\nabla l_{S}\left(\bzetaL \right),\mathbf{0} \right\}$ and plugging the quantities computed at the maximum value $\bv$ in (\ref{eq:lowerbound}) we get the bound
\begin{equation}
\label{eq:lowerBfinal}
0 \leq \lbzetaLS-\lbzetaL  \leq \frac{1}{2 c_{\bar{m}}} \left\|\nabla l_{S}^+\left(\bzetaL \right) \right\|^2.
\end{equation}

To lower bound the numerator, consider a single coordinate $j \in S$. It suffices to restrict to those coordinates $j$ where $\gradljbzetaL > 0$. Otherwise, by Lemma~4.4 $f\left(L \cup \{j\} \right) = f(L)$. Let $\bonej$ be a vector with a value one only at the $j^{th}$ coordinates and zero elsewhere. For a $\alpha \geq 0$, define $\byj = \bzetaL + \alpha \bonej$ such that $\left(\bzetaL, \byj \right) \in \tilde{\Omega}$. As $\bzetaLj$ is the optimal point for $f\left(L \cup \{j\} \right)$ we have 
\begin{align*}
\lbzetaLj - \lbzetaL &\geq l\left(\byj\right) - \lbzetaL \\
&\geq  \left\langle\nabla \lbzetaL,  \alpha \bonej \right\rangle - \frac{\tilde{C}_1}{2} \alpha^2.
\end{align*}

Maximizing w.r.t $\alpha$ we get $\alpha = \frac{\gradljbzetaL}{\tilde{C}_1} \geq 0$. Substituting this maximum value we get
\begin{align}
&\lbzetaLj - \lbzetaL \geq \frac{1}{2 \tilde{C}_1} \left(\gradljbzetaL\right)^2 \nonumber \\
\label{eq:upperBfinal}
\implies & \sum\limits_{j \in S} \left[\lbzetaLj - \lbzetaL \right] \geq \frac{1}{2 \tilde{C}_1} \left\|\nabla l_{S}^+\left(\bzetaL \right) \right\|^2.
\end{align}

From the equations~(\ref{eq:lowerBfinal}) and (\ref{eq:upperBfinal}) we get $\gamma_{L,S} \geq \frac{c_{\bar{m}}}{\tilde{C}_1}$. The minimum over all sets $L$, $S$ proves the theorem.

\subsection{Proof of Lemma~\ref{lem:KKTGradient}}
Recall that the optimization problem for computing the set function $f(L)$ requires that for $j \notin L$, $\x_j = 0$. Let $\lambda_j$ denote the corresponding Lagrange multiplier. The stationarity condition of the unconstrained problem implies that at the optimum $\bzetaL$, $\lambda_j = - \gradljbzetaL \geq 0$. In the optimization problem for computing $f\left(L \cup \{j\} \right)$, $\lambda_j$ is the KKT multiplier for the constraint $\x_j \geq 0$. As $\lambda_j$ satisfies the dual feasibility condition which together with other KKT conditions are both necessary and \emph{sufficient} for the optimality of maximizing concave functions $l$, we get $\bzetaLj = \bzetaL$.

\subsection{Proof of Theorem~\ref{pd}}
Let $L = L_i^D$ be the set chosen by the ProtoDash up to the iteration $i$ such that $L_m^D = L^D$ and $L^{\ast}$ be the optimal set. Define the residual set $L_R = L^{\ast} \setminus L$. Given $L$, let $v$ be the index that would be selected by running next step of ProtoDash. Let $D(i+1) = f(L \cup \{v\}) - f(L)$. Defining $\byv = \bzetaL + \alpha \bonev$ for some $\alpha \geq 0$ and recalling that $\bzetaLv$ is the maximizing point for $f(L \cup \{v\})$ we get
\begin{align*}
D(i+1) &\geq l\left( \byv \right) - \lbzetaL \\
&\geq  \left\langle\nabla \lbzetaL,  \alpha \bonev \right\rangle - \frac{\tilde{C}_1}{2} \alpha^2.
\end{align*}
Setting $\alpha = \frac{\gradlvpbzetaL}{\tilde{C}_1}$ we get 
\begin{align*}
D(i+1) \geq \frac{1}{2 \tilde{C}_1} \left[\gradlvpbzetaL \right]^2 \geq \frac{1}{2 m \tilde{C}_1} \sum\limits_{j \in L_R}  \left[\gradljpbzetaL \right]^2
\end{align*}
where the last inequality is a result of the fact that ProtoDash chooses the coordinate $v$ that maximizes  the gradient value $\nabla \lbzetaL$ and $\left|L_R\right| \leq m$. Letting $\bar{m} = |L|+\left|L_R\right| \leq 2m$, $B(i) = f\left(L^{\ast}\right) - f(L)$ and using (\ref{eq:lowerBfinal}) we find
\begin{align*}
m D(i+1)  \geq \frac{c_{\bar{m}}}{\tilde{C}_1} [f(L \cup L_R) - f(L)]  \geq \frac{c_{2m}}{\tilde{C}_1} B(i)
\end{align*}
where use the inequalities that $c_{\bar{m}} \geq c_{2m}$ and $L^{\ast} \subseteq L \cup L_R$. Setting $\kappa =\frac{c_{2m}}{\tilde{C}_1 m}$ and noting that $D(i+1) = B(i) - B(i+1)$ we get the recurrence relation $B(i+1) \leq (1-\kappa) B(i)$ which when iterated $i$ times starting from step $0$ gives $B(i) \leq (1-\kappa)^i B(0)$.  Plugging in $B(k) =  f\left(L^{\ast}\right) - f\left(L^D\right)$ and $B(0) = f\left(L^{\ast}\right)$ gives us the required inequality
\begin{equation*}
f\left(L^D\right) \geq f\left(L^{\ast}\right) \left[1-(1-\kappa)^m \right] \geq f\left(L^{\ast}\right) \left[1- e^{- \frac{c_{2m}}{\tilde{C}_1}} \right].
\end{equation*}

\subsection{Proof of Corollary~\ref{cor1}}
Let $L$ be the set chosen by the ProtoDash up to the current iteration. For every $j \notin L$, define the vector $\s_j$ of length $|L|$ whose $i^{th}$ element $s_{j,i} = K_{j,i}$ for $i \in L$. Let $\w^{\ast} = \bzetaL_L$, $\w^j = \bzetaLj_L$, $\bmu_{p,L}$ and $K_L$ be the restriction of the corresponding entities on the coordinates specified by $L$ and similarly let $w^j_j = \bzetaLj_j$. Recall that in the next iteration, ProtoDash chooses the prototype $j^D$ such that $j^D=\argmax\limits_{j}\gradljbzetaL =\argmax\limits_{j}\mu_{p,j}-\s_j^T\w^{\ast}$. Pursuant to Lemma~\ref{lem:KKTGradient} we have, if $\mu_{p,j}-\s_j^T\w^{\ast} \leq 0$, then $\bzetaLj = \bzetaL$ and specifically, $\w^j = \w^{\ast}$. Otherwise, the stationarity and complementary slackness KKT conditions entails that $w^j_j = \frac{\mu_{p,j}-\s_j^T\w^j}{K_j}$. Using this value of $w^j_j$, we see that the optimization problem that maximizes
\begin{align}
\label{eq:lj}
l_j(\w) = &-\frac{1}{2} \w^TK_L \w + \bmu_{p,L}^T \w+ \frac{1}{2K_j}\left(\mu_{p,j} - \s_j^T\w\right)^2 \\
&\text{subject to } \w \geq 0, \text{ and } \s_j^T\w \leq \mu_{p,j} \nonumber
\end{align}
attains its optimum at $\w = \w^j$. Particularly, $l_j(\w^j) \geq l_j\left(\w^{\ast}\right), \forall j$. The choice $j^D$ by our ProtoDash method has the property that $l_{j^D}(\w^{\ast}) \geq l_j(\w^{\ast})$ assuming that the prototypes are normalized so that their self-norm $K_j = 1, \forall j$, where as ProtoGreedy choose that index $j^G$ where $l_{j^G}(\w^{j^G}) \geq l_j(\w^j)$. Ergo, ProtoDash selects the prototype $j^D$ that maximizes the lower bound $l_{j^D}(\w^{\ast})$. To see that  $l_j\left(\w^{\ast}\right)$ is a tight lower bound for $l_j(\w^j)$, consider only the first to terms in the right hand side of (\ref{eq:lj}).
From the optimality of $\bzetaL$ we find $-\frac{1}{2} (\w^{\ast})^T K_L \w^{\ast} + \bmu_{p,L}^T \w^{\ast} \geq -\frac{1}{2}(\w^j)^TK_L \w^j + \bmu_{p,L}^T \w^j, \forall j$. 
Hence $\s_j^T\w^j \leq \s_j^T\w^{\ast} \leq \mu_{p,j}$. If $\s_j^T\w^{\ast} \approx \s_j^T\w^j$ or $\s_j^T\w^j \approx \mu_{p,j}$ then lower bound will be tight.

\eat{
We used the following 40 health questionnaires from the Center of Disease Control (CDC): Alcohol Use, Reproductive Health, Occupation, Dermatology, Sexual Behavior, Current Health Status, Income, Hospital Utilization \& Access to Care, Smoking - Secondhand Smoke Exposure, Blood Pressure \& Cholesterol, Diet Behavior \& Nutrition,Medical Conditions, Physical Activity, Weight History, Housing Characteristics, Oral Health, Volatile Toxicant (Subsample), Drug Use, Sleep Disorders, Immunization, Health Insurance, Physical Functioning, Kidney Conditions - Urology, Diabetes, Food Security, Mental Health - Depression Screener, Acculturation, Creatine Kinase,Smoking - Recent Tobacco Use, Consumer Behavior, Disability, Taste \& Smell, Cardiovascular Health, Preventive Aspirin Use, Smoking - Household Smokers, Early Childhood, Pesticide Use, Osteoporosis, Weight History - Youth, Hepatitis.
}
\bibliographystyle{IEEEtran}
\bibliography{proto}

\end{document}